\documentclass{article} 
\usepackage[final]{colm2026_conference}

\usepackage{microtype}
\usepackage{hyperref}
\usepackage{url}
\usepackage{booktabs}
\usepackage{graphicx}
\usepackage{booktabs}
\usepackage{multirow}
\usepackage{adjustbox}
\usepackage{amsmath}
\usepackage{amsthm}
\usepackage{algorithm}
\usepackage{algpseudocode}
\definecolor{mygray}{gray}{.92}
\usepackage{algpseudocode}
\usepackage{bm}
\usepackage{xcolor}
\usepackage[english]{babel}
\newtheorem{theorem}{Theorem}
\newtheorem{assumption}{Assumption}

\usepackage{lineno}
\usepackage{subcaption}
\definecolor{darkblue}{rgb}{0, 0, 0.5}
\hypersetup{colorlinks=true, citecolor=darkblue, linkcolor=darkblue, urlcolor=darkblue}
\newcommand{\myparagraph}[1]{\noindent\textbf{#1}\quad}

\title{Wiener Representation Filtering for VLM Hallucination Suppression}

\author{Ameen Ali\thanks{These authors contibuted equally to this work} , Tamim Zoabi$^*$, Lidor Brami \& Lior  Wolf \\
Tel Aviv University\\
\texttt{\{ameenali,tamimzoabi,lidorbrami\}@mail.tau.ac.il},~\texttt{wolf@cs.tau.ac.il} \\
}

\newtheorem{lemma}[theorem]{Lemma}

\begin{document}

\ifcolmsubmission
\linenumbers
\fi

\maketitle

\begin{abstract}
Vision-language models (VLMs) excel at open-ended captioning and visual QA but often describe objects, attributes, or relations absent from the image, a phenomenon known as object hallucination. We propose a {training-free, post-hoc representation editing technique} that operates in the representation space of the language backbone. The method performs a lightweight, one-time offline calibration on a modest paired dataset to estimate the required covariance structures, using only forward passes and empirical second-order statistics with no gradient updates or fine-tuning, after which the correction is absorbed directly into the model's existing weights.
By modeling hidden states as a superposition of truthful and hallucination-associated components, we derive a Wiener-type estimator whose optimal gains are given in closed form from the covariances of paired truthful and hallucinated representations.
An eigendecomposition yields mode-wise attenuation that respects a stability criterion, i.e., the filter responds continuously to estimation noise.
The correction is applied once to the feed-forward output projections of selected deeper layers, at inference time, the model runs unchanged and at the same speed.
Experiments on LLaVA-1.5, MiniGPT-4, Gemma3, and mPLUG-Owl2 demonstrate consistent reductions in object hallucination on CHAIR, POPE, and MME while maintaining caption fluency and overall response quality. We further demonstrate the generality of our approach on the TempCompass video understanding benchmark and on discrete diffusion language models for grounded dialogue, showing that representation filtering reduces hallucinations even in temporal video reasoning and multi-step, sequence-wide denoising settings.
\end{abstract}

\section{Introduction}
\label{sec:intro}

Vision-language models (VLMs) have achieved remarkable success in bridging visual perception and natural language processing~\citep{liu2023visual,ye2023mplugowl2,ye2023mplugowl,chen2023minigptv2,zhu2023minigpt}, powering diverse applications such as image captioning, visual question answering, and multimodal reasoning. By integrating pre-trained vision encoders with large language models (LLMs), these systems generate fluent, context-aware descriptions of visual scenes. Despite their prowess, VLMs remain plagued by \emph{object hallucinations} (OH): the fabrication of objects, attributes, or spatial relations absent from the input image, often stemming from linguistic priors or distributional biases in training data \citep{huang2024visual,rohrbach-etal-2018-object,bai2024hallucination,zhong-etal-2024-investigating,huang2024visual}.

Existing mitigation strategies fall into two categories: decoding-time, such as constrained beam search \citep{freitag-al-onaizan-2017-beam} or logit debiasing~\citep{chuang2024dola,leng2024mitigating}, which modulate generation to promote grounded outputs, and post-hoc refinements such as Woodpecker \citep{yin2024woodpecker}, LURE \citep{zhou2024analyzing}), and the reference-guided editing method HALC \citep{chen2024halc}. While these yield gains in controlled scenarios, they introduce substantial overhead, often 5--10x inference slowdowns, demand task-specific adaptations, or rely on external resources.
\begin{figure*}[t]
    \centering
    \begin{subfigure}[t]{0.40\textwidth}
        \centering
        \includegraphics[width=\linewidth, height=0.75\textwidth]{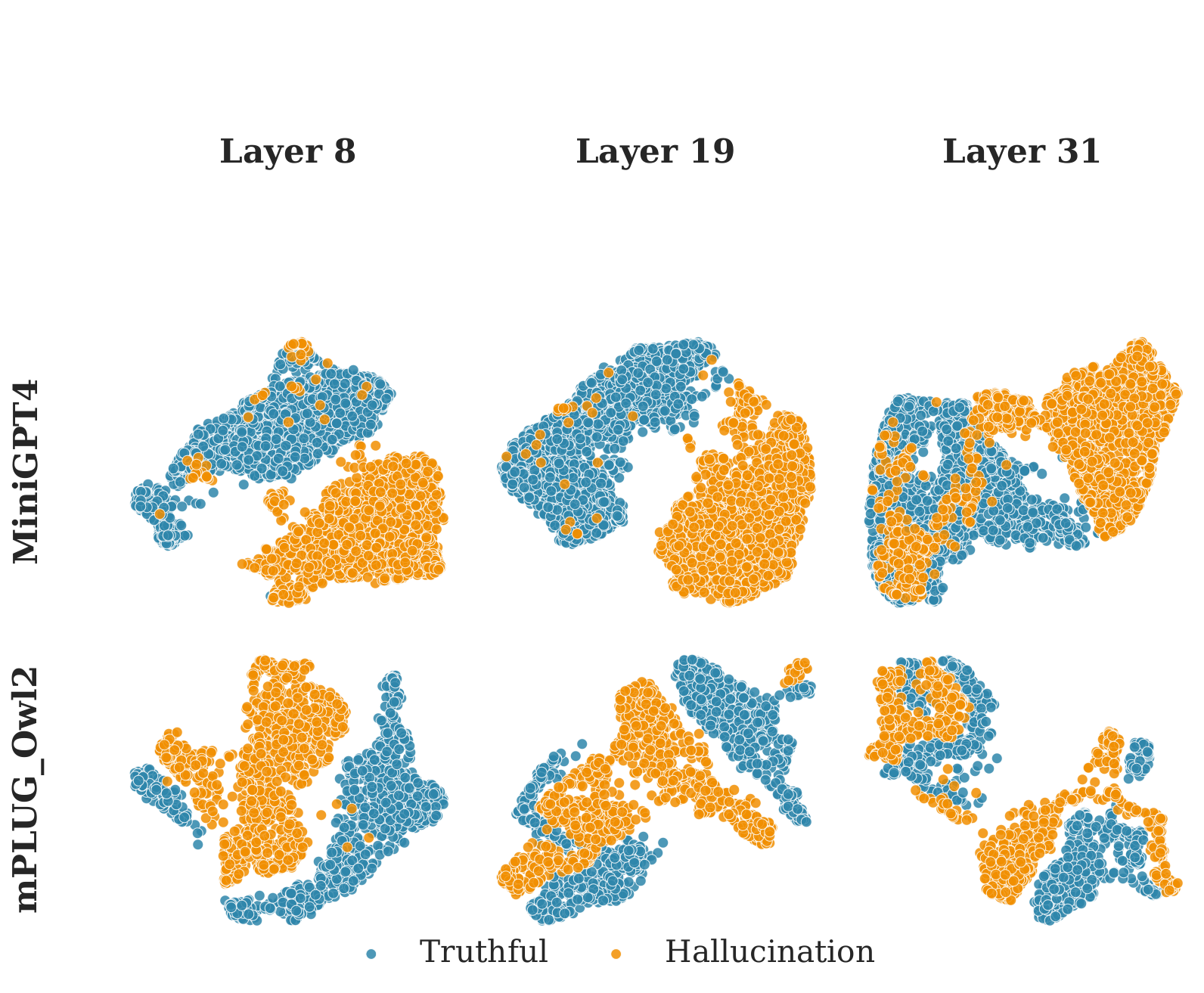}
        \subcaption*{(a)}
        \label{fig:tsne}
    \end{subfigure}
    \hfill
    \begin{subfigure}[t]{0.55\textwidth}
        \centering
        \includegraphics[width=\linewidth,height=0.5\textwidth]{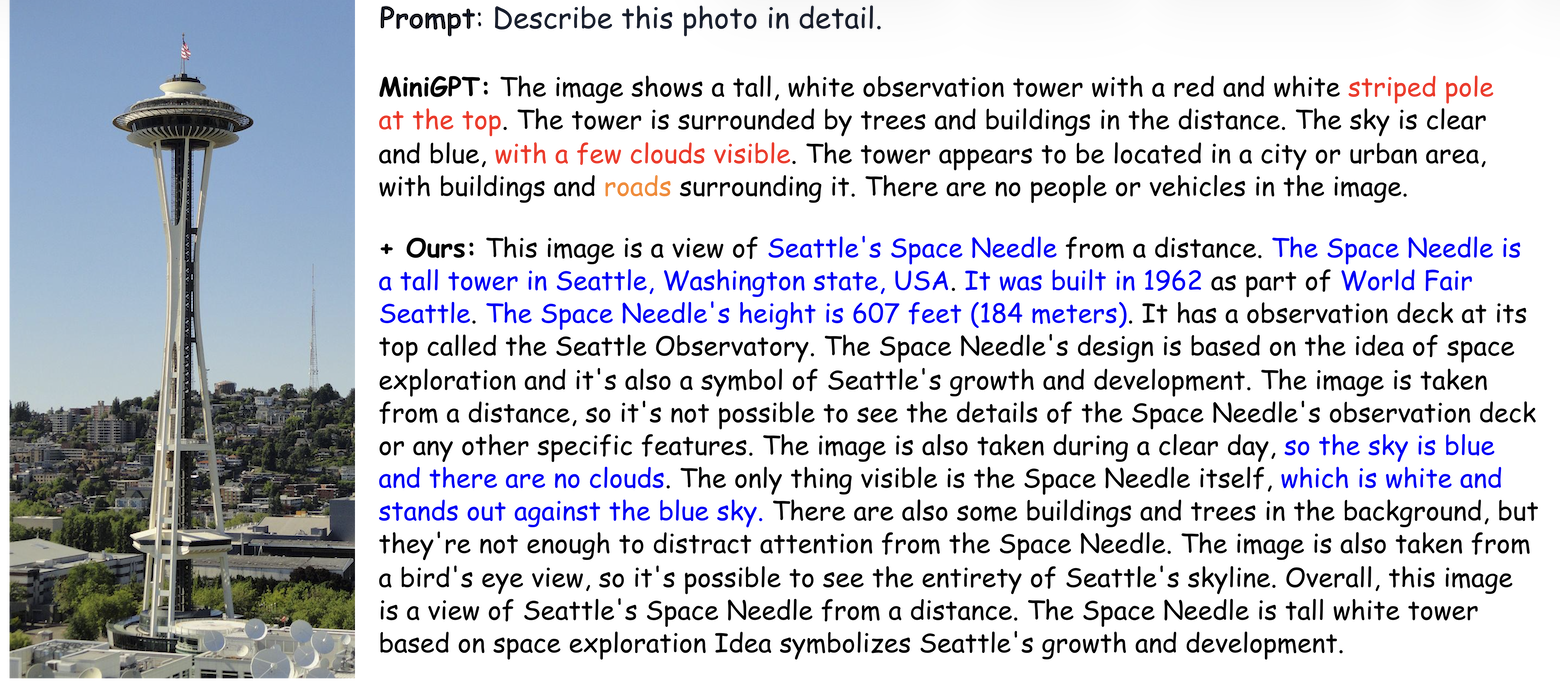}
        \subcaption*{(b)}
        \label{fig:qual_1}
    \end{subfigure}
    \caption{(a) UMAP visualization of MiniGPT4 and mPLUG-Owl2 hidden activations for truthful (blue) and hallucinatory (amber) samples from LURE, showing distinct clusters that reveal low-rank hallucination subspaces. (b) Qualitative comparison of captions generated by MiniGPT before and after applying our proposed Wiener based correction. The baseline model hallucinates non-existent objects such as clouds, and striped pole at the top, while our method suppresses these errors.}
    \label{fig:teaaser}
\end{figure*}
We address this gap with a gradient-free and inference-efficient correction derived from a representation-space estimation framework. 
We model hidden states as the superposition of a truthful component and a hallucination-associated distortion, each characterized by second-order statistics estimated from paired truthful and hallucinatory generations. 
Under this formulation, hallucination mitigation reduces to a linear minimum mean-square error (MMSE) estimation problem.

Empirically, we observe structured geometric separation between truthful and hallucinatory representations in deeper layers. 
As shown in Fig.~\ref{fig:teaaser}(a), UMAP projections reveal distinct latent regions for truthful (blue) and hallucinatory (amber) states, suggesting that hallucination manifests as anisotropic distortion in representation space.

The MMSE formulation yields a closed-form Wiener estimator whose operator depends on the covariances of truthful and hallucination components. Solving a generalized eigenvalue problem identifies spectral modes ranked by distortion-to-signal ratios, which are smoothly reweighted by Wiener gains to produce covariance-driven filtering in representation space. This continuous formulation ensures stability: small covariance estimation errors induce proportionally bounded changes in the operator.

The resulting transformation acts as a spectral equalizer in latent space. We estimate it once from a calibration set and apply it offline to the feed-forward output projections of selected deeper LLM layers. A single sharpness exponent~$\alpha$ controls attenuation strength, suppressing hallucination-dominated modes while preserving semantically grounded directions. Because the edit is absorbed into existing weights, the architecture, parameter count, and inference cost remain unchanged.

Inspired by classical Wiener filtering in signal processing~\citep{oppenheim1999discrete}, our approach interprets hallucination as structured spectral distortion and corrects it through covariance-driven spectral shaping. 
Evaluations on LLaVA-1.5~\cite{liu2023visual}, MiniGPT-4~\citep{zhu2023minigpt}, and mPLUG-Owl2~\citep{ye2023mplugowl2} across CHAIR and POPE benchmarks demonstrate consistent reductions in object hallucination while preserving fluency and visual grounding. A qualitative comparison in Fig.~\ref{fig:teaaser} (b) further illustrates Wiener’s effect, showing how the same model, when filtered through our method, eliminates hallucinated objects and yields faithful, visually grounded descriptions.
\section{Related Works}
\textbf{Vision-Language Models}
VLMs combine visual and textual modalities for tasks such as image captioning and VQA. They typically consist of a pre-trained vision encoder (e.g., CLIP-ViT~\citep{radford2021learning}) to extract image features, a projection module that maps these features into the LLM embedding space, and a large language model (e.g., LLaMA~\citep{touvron2023llama} or Vicuna~\citep{chiang2023vicuna}), which is usually frozen during alignment and later instruction-tuned on multimodal datasets.

Modern architectures improve modularity and efficiency through three integration strategies: early fusion~\citep{gao2023llama,zhang2023llamaadapter,liu2023visual,bai2023qwenvlversatilevisionlanguagemodel}, bridging~\citep{dai2023instructblip}, and mid-fusion~\citep{alayrac2022flamingo}. 
Early fusion prepends projected visual tokens to the LLM input, as in LLaVA~\citep{liu2023visual} and Qwen-VL2~\citep{wang2024qwen2}. 
Bridging compresses visual tokens into a small set of query embeddings, used in BLIP~\citep{dai2023instructblip} and mPLUG-Owl~\citep{ye2023mplugowl}. 
Mid-fusion distributes cross-attention across layers, as in Flamingo~\citep{alayrac2022flamingo} and LLaMA 3.2-Vision~\citep{dubey2024llama}, enabling deeper multimodal interaction. 

\textbf{Object Hallucinations in VLMs}
Visual hallucinations in vision large language models refer to the generation of plausible but factually incorrect outputs that misrepresent the input visual content~\citep{liu2024survey,li-etal-2023-evaluating,kim2024discovering}, encompassing a range of distortions from subtle inaccuracies to overt fabrications. 
These include inventing object attributes (e.g., assigning colors, sizes, or states not depicted, such as describing a grayscale image as vibrant) or fabricating spatial relationships. 
In multimodal tasks like visual question answering~\citep{LiaZeh_VisionAmplified_MICCAI2025,10.1145/3664647.3681663,zhu2024combating} and image captioning~\citep{rohrbach-etal-2018-object,biten2022let,zhai2023halle,li-etal-2023-evaluating}, hallucinations typically appear as overly assured claims that merge observed visual elements with embedded textual knowledge, resulting in responses that appear relevant to the prompt on the surface but compromise reliability in scenarios demanding accurate visual interpretation. 

To mitigate hallucinations in VLMs, methods are broadly classified into training-based~\citep{sun2024aligning,jing2024fgaif,zhao2025looking} and training-free~\citep{huang2024opera,leng2024mitigating,chen2024halc,yang2025nullu} categories. 
Training-based techniques involve retraining components to strengthen vision-language alignment, such as fine-tuning the projection module or LLM with hallucination-aware losses that reward fidelity to input visuals, often leveraging datasets augmented with paired image-text contrasts. 
Additional strategies include reinforcement learning from human feedback~\citep{sun2024aligning} tailored to multimodal outputs. 
Recently, \citet{jing2024fgaif} introduced FGAIF, which provides fine-grained AI-generated feedback for hallucination supervision (object existence, attributes, and relations), enabling more precise RL-based alignment of LVLMs.  

Among training-free approaches, \citet{huang2024opera} proposed OPERA, which applies an over-trust penalty and a retrospection-allocation mechanism during decoding to reduce hallucinations. 
\citet{leng2024mitigating} introduced Visual Contrastive Decoding, which contrasts outputs on original versus distorted visuals to suppress object hallucinations. 
HALC~\citep{chen2024halc} is an adaptive focal-contrast decoding strategy that corrects hallucinated tokens on the fly. 
Similarly, Nullu~\citep{yang2025nullu} identifies hallucination-related directions by performing an SVD over feature differences and then removes the top singular vectors via hard projection, fully excising those directions from the model’s representation space. 

In contrast, our method formulates hallucination mitigation as a linear minimum mean-square error (MMSE) estimation problem.
We model hidden states as a combination of truthful signal and hallucination-induced distortion and estimate their respective covariances from calibration data.
Using this second-order characterization, we compute a closed-form Wiener operator that proportionally attenuates hallucination-dominated modes based on their distortion-to-signal ratios, preserving semantically meaningful directions.
This approach provides stable correction without removing any component entirely, maintaining the integrity of the model’s internal representations while reducing hallucinations.

{An earlier version of this work~\citep{ali2024suppressing} proposed suppressing hallucination-related directions via low-pass spectral filter. The present work substantially extends that formulation by introducing a Wiener-filter-based representation editing framework that performs smooth, covariance-aware attenuation rather than a spectral truncation. This generalization explicitly models both signal and noise statistics, yielding a more principled and flexible mechanism for suppressing hallucination-related components while preserving task-relevant information.}

\section{Method}
This section introduces the VLM background and notation, analyzes the covariance structure of hallucinated representations, derives the Wiener spectral filter and its stability guarantees, and describes how the resulting operator is absorbed into the feed-forward projections of deeper transformer layers.
\subsection{Background and Notation}
\myparagraph{Notation.} Throughout the paper, vectors are denoted using bold lowercase letters (e.g., $\mathbf{x}, \mathbf{s}, \mathbf{n}$) and matrices using bold uppercase letters (e.g., $\mathbf{W}$). Scalars are in standard italic (e.g., $d, \alpha, \lambda_j$). The operator norm is $\|\cdot\|_2$, and for a symmetric matrix $\mathbf{C}$, its spectrum is $\sigma(\mathbf{C})$.

\myparagraph{Background} Modern vision-language models (VLMs) process multimodal inputs through a three-stage pipeline. 
A vision encoder $\phi_v: \mathcal{I} \rightarrow \mathbb{R}^{n \times d_v}$ extracts $n$ visual tokens of dimension $d_v$. 
A projection layer $\phi_p: \mathbb{R}^{d_v} \rightarrow \mathbb{R}^{d}$ maps these tokens into the language model embedding space of dimension $d$. 
The combined visual-text sequence is then processed by a large language model (LLM) with $L$ transformer layers.

At each layer $\ell \in \{1,\dots,L\}$, hidden states are updated via self-attention and feed-forward networks (FFNs). 
The FFN block is defined as
\begin{equation}
\mathrm{FFN}_{\ell}(\mathbf{x}) 
= \mathbf{W}_{\ell}^{\text{out}} 
\sigma(\mathbf{W}_{\ell}^{\text{in}} \mathbf{x} + \mathbf{b}_{\ell}^{\text{in}}) 
+ \mathbf{b}_{\ell}^{\text{out}},
\end{equation}
where 
$\mathbf{W_{\ell}^{\text{in}}} \in \mathbb{R}^{d_{\text{ff}} \times d}$ ,$\mathbf{W_{\ell}^{\text{out}}} \in \mathbb{R}^{d \times d_{\text{ff}}}$, $\mathbf{b}_{\ell}^{\text{in}} \in \mathbb{R}^{d_{\text{ff}}}$ and $\mathbf{b}_{\ell}^{\text{out}} \in \mathbb{R}^{d}$.

Let $\mathbf{x} \in \mathbb{R}^d$ denote a representation vector with covariance
\begin{equation}
\boldsymbol{\Sigma} = \mathbb{E}[(\mathbf{x}-\boldsymbol{\mu})(\mathbf{x}-\boldsymbol{\mu})^\top].
\end{equation}
For symmetric $\boldsymbol{\Sigma}$, let $\boldsymbol{\Sigma} = \mathbf{Q}\boldsymbol{\Lambda} \mathbf{Q}^\top$ denote its eigendecomposition, where $\{\mathbf{q_j}\}$ are orthogonal spectral modes with variances $\{\lambda_j\}$. Spectral operators modify the covariance structure via functions applied to eigenvalues,
\begin{equation}
g(\boldsymbol{\Sigma}) = \mathbf{Q} g(\boldsymbol{\Lambda}) \mathbf{Q}^\top .
\end{equation}
\subsection{Hallucination Covariance Structure}

Object hallucination corresponds to systematic discrepancies between truthful and hallucinated internal representations.
Empirical visualization using UMAP (Fig.~\ref{fig:teaaser} (a)) reveals structured separation between truthful and hallucinated hidden states in deeper layers, indicating non-random geometric divergence.

To characterize this divergence quantitatively, we utilize a calibration dataset 
$\mathcal{D} = \{(I_i, c_i^+, c_i^-)\}_{i=1}^N$, 
where $c_i^+$ is a hallucinated caption generated by the model and $c_i^-$ is a human-verified truthful caption describing the same image $I_i$.
For each transformer layer $\ell$, we extract sequence-averaged representations:
\begin{equation}
\mathbf{x_i^+} = \frac{1}{T_i^+} \sum_{t=1}^{T_i^+} \mathbf{h}_{\ell,t}(I_i,c_i^+), 
\quad
\mathbf{x_i^-} = \frac{1}{T_i^-} \sum_{t=1}^{T_i^-} \mathbf{h}_{\ell,t}(I_i,c_i^-).
\end{equation}

We model hallucination as an additive perturbation in representation space, reflecting the empirical observation that hallucinated captions produce structured shifts relative to truthful representations while preserving much of the underlying semantic signal:
\begin{equation}
\mathbf{x}_i^+ = \mathbf{x}_i^- + \mathbf{d}_i,
\end{equation}
where $\mathbf{x}_i^-$ denotes the truthful representation and $\mathbf{d}_i$ captures hallucination-induced variation.
We estimate two covariance matrices:
\begin{align}
\boldsymbol{\Sigma}_T &= \frac{1}{N} \sum_{i=1}^N (\mathbf{x}_i^- - \boldsymbol{\mu}_T)(\mathbf{x}_i^- - \boldsymbol{\mu}_T)^\top, &
\boldsymbol{\Sigma}_H &= \frac{1}{N} \sum_{i=1}^N \textbf{d}_i \textbf{d}_i^\top.
\end{align}
Here, $\boldsymbol{\Sigma}_T$ captures second-order structure of truthful representations,
while $\boldsymbol{\Sigma}_H$ captures structured distortion induced by hallucinations.
Highly anisotropic spectra (Fig.~\ref{fig:spectra}) indicate that hallucination manifests as directional variance amplification rather than isotropic noise.

\subsection{Wiener Spectral Filtering in Latent Space}
We interpret hidden representations as $
\mathbf{h} = \mathbf{s} + \mathbf{n}$,
where $\mathbf{s}$ is the truthful component with covariance $\boldsymbol{\Sigma}_T$ and 
$\mathbf{n}$ is hallucination-associated distortion with covariance $\boldsymbol{\Sigma}_H$. 
We assume that the truthful component $\mathbf{s}$ and hallucination distortion $\mathbf{n}$ are zero-mean and mutually uncorrelated, i.e.,
\begin{assumption}
\label{ass:signal_model}
The truthful component $\mathbf{s}$ and hallucination distortion $\mathbf{n}$ are zero-mean and mutually uncorrelated:
\begin{equation}
    \mathbb{E}[\mathbf{s}] = \mathbf{0}, \quad \mathbb{E}[\mathbf{n}] = \mathbf{0}, \quad \mathbb{E}[\mathbf{s}\mathbf{n}^\top] = \mathbf{0}.
\end{equation}
\end{assumption}
Under Assumption~\ref{ass:signal_model} and the additive model $\mathbf{h} = \mathbf{s} + \mathbf{n}$, and since $\mathbf{s}$ and $\mathbf{n}$ are uncorrelated, the covariance of the observed representation decomposes as
\begin{equation}
    \mathrm{Cov}(\mathbf{h}) = \boldsymbol{\Sigma}_T + \boldsymbol{\Sigma}_H.
\end{equation}
Standard LMMSE theory then gives the Wiener operator:
\begin{equation}
\mathbf{A}^\star = \boldsymbol{\Sigma}_T(\boldsymbol{\Sigma}_T + \boldsymbol{\Sigma}_H)^{-1}.
\end{equation}

{The additive model is not a literal claim that the transformer factorizes hidden states into independent semantic components; it is a paired-residual construction induced by the calibration data. For each image and layer we pair a hallucinated representation $\mathbf{x}_i^+$ with the truthful representation $\mathbf{x}_i^-$ of the same image and set $\mathbf{d}_i = \mathbf{x}_i^+ - \mathbf{x}_i^-$. After centering ($\mathbf{s}_i = \mathbf{x}_i^- - \boldsymbol{\mu}_T$, $\mathbf{n}_i = \mathbf{d}_i - \boldsymbol{\mu}_d$, $\mathbf{h}_i = \mathbf{x}_i^+ - \boldsymbol{\mu}_+$) and using $\boldsymbol{\mu}_+ = \boldsymbol{\mu}_T + \boldsymbol{\mu}_d$, we obtain $\mathbf{h}_i = \mathbf{s}_i + \mathbf{n}_i$ with zero mean. The only substantive approximation is the weak cross-covariance assumption $\mathbf{C} = \mathbb{E}[\mathbf{s}\mathbf{n}^\top] \approx \mathbf{0}$; retaining it gives the exact correlated LMMSE estimator $\mathbf{A}_C^\star = (\boldsymbol{\Sigma}_T + \mathbf{C})(\boldsymbol{\Sigma}_T + \boldsymbol{\Sigma}_H + \mathbf{C} + \mathbf{C}^\top)^{-1}$, and our estimator is its $\mathbf{C} \approx \mathbf{0}$ case, validated empirically in Section~\ref{sec:assumption_validation}.}

To obtain an interpretable spectral form, we express the estimator in the eigenbasis of the hallucination covariance $\boldsymbol{\Sigma}_H = \mathbf{Q} \boldsymbol{\Lambda} \mathbf{Q}^\top$, whose spectrum empirically reveals a small set of dominant distortion directions (Appendix.~\ref{app:spectrum}), yielding eigenmodes $\{\mathbf{q}_j\}$ with distortion variances $\{\lambda_j\}$. In this basis, the signal variance in mode $j$ is $\tau_j^2 = \mathbf{q}_j^\top \boldsymbol{\Sigma}_T \mathbf{q}_j$. The distortion-to-signal ratio in mode $j$ is $\nu_j = \lambda_j / \tau_j^2$ (with $\tau_j^2 > 0$). Mode-wise Wiener gains are
\begin{equation}
\gamma_j = \frac{1}{1+\nu_j} = \frac{\tau_j^2}{\tau_j^2 + \lambda_j}.
\end{equation}

The standard Wiener gains $\gamma_j$ provide the MMSE-optimal attenuation, but in practice a tunable shaping strength is desirable to compensate for covariance estimation error and distributional mismatch between the calibration set and deployment data. We therefore generalize via a sharpness exponent $\alpha > 0$:
\begin{equation}
\tilde{\gamma}_j = (1+\nu_j)^{-\alpha} = \left( \frac{\tau_j^2}{\tau_j^2 + \lambda_j} \right)^\alpha.
\end{equation}

The resulting filter is
\begin{equation}
\mathbf{F}_{\alpha} = \mathbf{Q}\,\mathrm{diag}(\tilde{\gamma}_1,\dots,\tilde{\gamma}_d)\,\mathbf{Q}^\top.
\end{equation}

{The mode-wise filter $\mathbf{F}_\alpha$ is a diagonalized Wiener approximation in the eigenbasis of $\boldsymbol{\Sigma}_H$, not generally an exact rewriting of the full operator $\boldsymbol{\Sigma}_T(\boldsymbol{\Sigma}_T + \boldsymbol{\Sigma}_H)^{-1}$; it becomes exact when $\boldsymbol{\Sigma}_T$ and $\boldsymbol{\Sigma}_H$ are jointly diagonalizable. We use the eigenvectors of $\boldsymbol{\Sigma}_H$ to identify hallucination-dominated directions and compute per-mode truthful variance $\tau_j^2 = \mathbf{q}_j^\top \boldsymbol{\Sigma}_T \mathbf{q}_j$, yielding a stable approximation chosen for interpretability and numerical robustness.}

Wiener filtering shapes modes according to their distortion-to-signal ratio. This continuity ensures stability: small perturbations in the covariance estimates induce proportionally bounded changes in $\mathbf{F}_\alpha$.
\myparagraph{Stability of Spectral Shaping} Since the spectral filter is constructed from empirical covariance estimates, it is important that the resulting operator remains stable under estimation noise. The following lemma shows that operator-Lipschitz spectral maps guarantee this property, ensuring that small covariance perturbations do not produce large changes in the resulting representation operator.

\begin{lemma}
\label{lem:spectral_stability}
Let $\mathbf{C}, \widehat{\mathbf{C}} \in \mathbb{R}^{d\times d}$ be symmetric, and let $g$ be a scalar function defined on an interval $I\subset \mathbb{R}$ that contains $\sigma(\mathbf{C}) \cup \sigma(\widehat{\mathbf{C}})$. Assume $g$ is \emph{operator-Lipschitz} on $I$, i.e., there exists $K>0$ such that for all symmetric $\mathbf{A},\mathbf{B}$ with $\sigma(\mathbf{A}),\sigma(\mathbf{B})\subset I$,
\begin{equation}
\|g(\mathbf{A})-g(\mathbf{B})\|_2 \le K\,\|\mathbf{A}-\mathbf{B}\|_2.
\end{equation}
Then $g$ is stable in the sense that
\begin{equation}
\|g(\widehat{\mathbf{C}})-g(\mathbf{C})\|_2 \le K\,\|\widehat{\mathbf{C}}-\mathbf{C}\|_2.
\end{equation}
\end{lemma}
Proof is provided in Appendix.~\ref{app:proof}.

{In our implementation the gains satisfy $0 \le f_j \le 1$, so the edit is non-expansive in the hallucination eigenbasis and amplifies no spectral mode. Rather than hard rank truncation, all eigenmodes are retained and attenuated continuously, reducing sensitivity to unstable eigenvector ordering. We clamp tiny negative $\tau_j^2$ to zero and leave modes unchanged when $\tau_j^2 + \lambda_j$ is numerically zero. The sensitivity analysis in Appendix~\ref{app:sensitivity_alpha} shows performance remains stable over a broad range of $\alpha$.}
\myparagraph{Interpretation.} The mode-wise gains are set by the distortion-to-signal ratios derived from the covariance structure, and $\alpha$ acts only as a \emph{sharpness} control that adjusts how strongly these gains are expressed without changing their ordering. The transformation therefore remains stable and model-adaptive, governed by covariance geometry rather than downstream tuning, and admits a broad operating range for $\alpha$, typically on the order of \(1\text{--}10^{2}\) (Appendix~\ref{app:sensitivity_alpha}).

\subsection{Weight Correction and Inference}

We incorporate the filter into a selected subset of deeper layers $\mathcal{L} \subseteq \{1,\dots,L\}$, where hallucination-induced separation is empirically most pronounced (see Section~\ref{sec:ablations} for an ablation over layer depth), by modifying the FFN output projection
$\widetilde{W}_{\ell}^{\text{out}}
= F_{\alpha}\,W_{\ell}^{\text{out}}$, 
where $F_{\alpha} \in \mathbb{R}^{d \times d}$ acts on the output feature dimension.

This edit corresponds to applying the Wiener estimator in representation space.
Because $\mathbf{F}_{\alpha}$ is precomputed offline and absorbed into the weights,
no additional computation is introduced at inference time.
The architecture, and runtime remain unchanged.
\section{Experiments}
\begin{figure*}[t]
\centering
\includegraphics[width=1\textwidth]{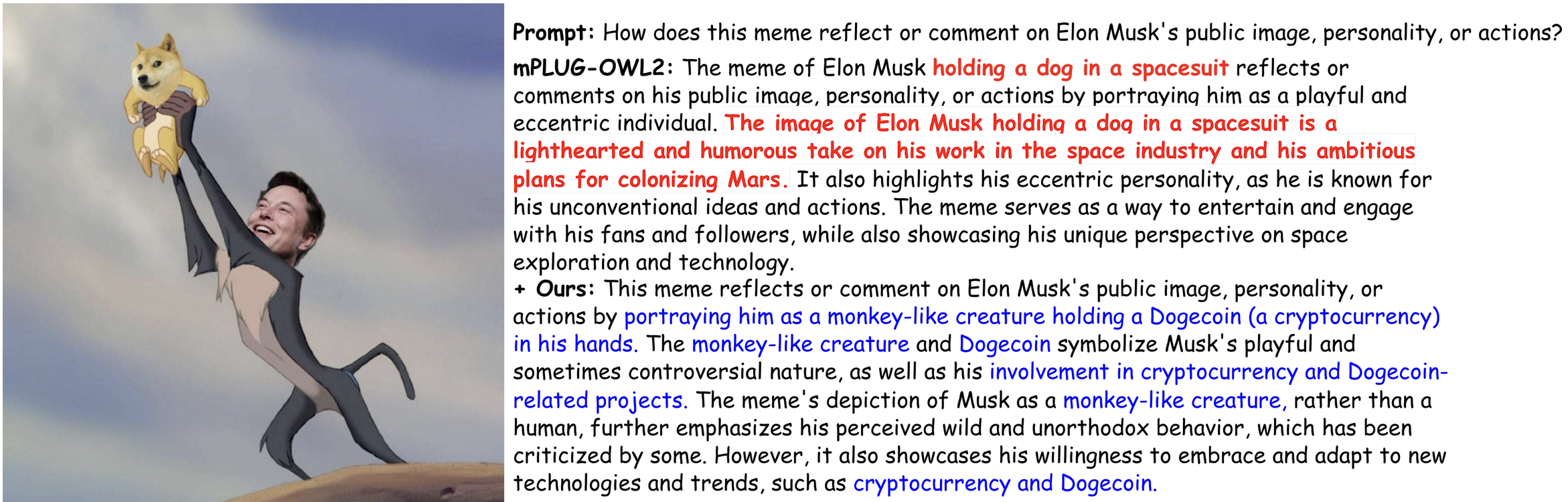}
 \caption{A qualitative case study comparing model-generated descriptions for a complex meme. We visualize the outputs from the  Baseline and our method, color-coding text segments as \textcolor{red}{Hallucinations} and \textcolor{blue}{Truth}.}\label{fig:qual_results}
\vspace{-7pt}
\end{figure*}

\noindent\textbf{Evaluation Datasets\quad} We evaluate the proposed Wiener-based correction on three complementary hallucination-centric benchmarks. (1) CHAIR~\citep{rohrbach-etal-2018-object} measures hallucination frequency by comparing generated captions with MSCOCO~\citep{lin2014microsoft} object annotations, reporting $\text{CHAIR}_S$ (sentence-level rate) and $\text{CHAIR}_I$ (object-level rate); lower is better, over 500 MSCOCO validation images averaged across three seeds. (2) POPE~\citep{li2023evaluating} is a VQA-style yes/no benchmark probing object presence, with three negative-sampling strategies (random, popular, and adversarial); we report accuracy, precision, and F1 on 500 MSCOCO images with ten queries each. (3) MME~\citep{fu2023mme} is a broader diagnostic suite for multimodal reasoning and visual grounding, complementing CHAIR and POPE across object presence, attributes, and spatial relationships.

\textbf{Vision-Language Models.}
We evaluate the proposed Wiener correction on three open-source VLMs representing different vision-language fusion paradigms. 
\citet{liu2023visual} present LLaVA-1.5, which adopts early fusion, concatenating \citet{radford2021learning} CLIP visual tokens with text embeddings before processing with Vicuna-7B, emphasizing modularity and efficiency. 
\citet{zhu2023minigpt} introduce MiniGPT-4 which follows a bridging design, where Q-Former–compressed visual features are injected into LLaMA-7B, reducing token dimensionality while preserving alignment. 
\citet{ye2023mplugowl2} propose mPLUG-Owl2 which employs a mid-fusion strategy with modular connectors and dual-branch pathways to balance visual understanding and text generation. 
These architectures span early, bridged, and mid-level fusion mechanisms, enabling evaluation across diverse representation geometries.

\textbf{Baselines.}
We compare against representative hallucination-mitigation approaches. Decoding-based strategies operate at inference time without modifying parameters: beam search~\citep{freitag2017beam}; Visual Contrastive Decoding (VCD)~\citep{leng2024mitigating}, which contrasts logits under original and perturbed visual inputs; OPERA~\citep{huang2024opera}, which adds an over-trust penalty and a retrospection mechanism; HALC~\citep{chen2024halc}, an adaptive focal-contrast decoding strategy; and MARINE~\citep{zhao2025mitigating}, which refines generation using visual-text alignment signals. Weight-editing approaches instead modify parameters offline: Nullu~\citep{yang2025nullu} identifies hallucination-associated directions via eigendecomposition of feature differences and removes dominant components through orthogonal projection.

\begin{table}[t]
\centering

\setlength{\tabcolsep}{4pt}
\renewcommand{\arraystretch}{1.15}
\begin{adjustbox}{max width=\textwidth}
\begin{tabular}{@{}l@{~}|@{~}c@{~}c@{~}c@{~}|@{~}c@{~}c@{~}c@{~}|@{~}c@{~}c@{~}c@{}}
\toprule
\textbf{Method} 
& \multicolumn{3}{c|}{\textbf{LLaVA-1.5}} 
& \multicolumn{3}{c|}{\textbf{MiniGPT-4}} 
& \multicolumn{3}{c}{\textbf{mPLUG-Owl2}} \\
\cmidrule(lr){2-4} \cmidrule(lr){5-7} \cmidrule(lr){8-10}
& CHAIR$_S$$\downarrow$ & CHAIR$_I$$\downarrow$ & BLEU$\uparrow$
& CHAIR$_S$$\downarrow$ & CHAIR$_I$$\downarrow$ & BLEU$\uparrow$
& CHAIR$_S$$\downarrow$ & CHAIR$_I$$\downarrow$ & BLEU$\uparrow$ \\
\midrule
Greedy
& 18.87{\scriptsize$\pm$1.68} & 6.03{\scriptsize$\pm$0.61} & 0.158{\scriptsize$\pm$0.002}
& 21.20{\scriptsize$\pm$1.60} & 7.40{\scriptsize$\pm$0.52} & 0.156{\scriptsize$\pm$0.002}
& 23.33{\scriptsize$\pm$1.97} & 8.50{\scriptsize$\pm$0.96} & 0.151{\scriptsize$\pm$0.002} \\
Beam Search~\cite{freitag2017beam}
& 17.40{\scriptsize$\pm$1.00} & 5.67{\scriptsize$\pm$0.46} & 0.159{\scriptsize$\pm$0.001}
& 21.47{\scriptsize$\pm$3.23} & 7.30{\scriptsize$\pm$0.95} & 0.160{\scriptsize$\pm$0.001}
& 19.60{\scriptsize$\pm$0.20} & 7.27{\scriptsize$\pm$0.31} & 0.154{\scriptsize$\pm$0.001} \\
DoLa~\cite{chuang2024dola}
& 19.33{\scriptsize$\pm$1.29} & 6.03{\scriptsize$\pm$0.35} & 0.158{\scriptsize$\pm$0.002}
& 21.53{\scriptsize$\pm$1.66} & 7.67{\scriptsize$\pm$0.68} & 0.156{\scriptsize$\pm$0.002}
& 21.87{\scriptsize$\pm$1.89} & 8.77{\scriptsize$\pm$0.72} & 0.151{\scriptsize$\pm$0.001} \\
VCD~\cite{leng2024mitigating}
& 23.87{\scriptsize$\pm$2.72} & 7.70{\scriptsize$\pm$0.90} & 0.144{\scriptsize$\pm$0.010}
& 21.13{\scriptsize$\pm$2.25} & 7.63{\scriptsize$\pm$0.68} & 0.154{\scriptsize$\pm$0.001}
& 27.07{\scriptsize$\pm$1.22} & 10.23{\scriptsize$\pm$0.55} & 0.139{\scriptsize$\pm$0.003} \\
OPERA~\cite{huang2024opera}
& 17.60{\scriptsize$\pm$2.09} & 5.63{\scriptsize$\pm$0.42} & 0.160{\scriptsize$\pm$0.003}
& 21.73{\scriptsize$\pm$2.04} & 7.63{\scriptsize$\pm$0.47} & 0.173{\scriptsize$\pm$0.002}
& 17.47{\scriptsize$\pm$0.61} & 7.33{\scriptsize$\pm$0.72} & 0.151{\scriptsize$\pm$0.002} \\
MARINE~\cite{zhao2025mitigating}
& 18.27{\scriptsize$\pm$0.76} & 5.90{\scriptsize$\pm$0.36} & 0.155{\scriptsize$\pm$0.002}
& 20.27{\scriptsize$\pm$2.42} & 7.27{\scriptsize$\pm$0.75} & 0.152{\scriptsize$\pm$0.001}
& -- & -- & -- \\
HALC~\cite{chen2024halc}
& 16.13{\scriptsize$\pm$0.99} & 5.23{\scriptsize$\pm$0.47} & 0.159{\scriptsize$\pm$0.002}
& 17.67{\scriptsize$\pm$1.40} & 6.57{\scriptsize$\pm$0.84} & 0.175{\scriptsize$\pm$0.002}
& 17.27{\scriptsize$\pm$0.58} & 7.10{\scriptsize$\pm$0.61} & 0.153{\scriptsize$\pm$0.001} \\
Nullu~\cite{yang2025nullu} (Greedy)
& 19.00{\scriptsize$\pm$1.25} & 5.93{\scriptsize$\pm$0.57} & 0.152{\scriptsize$\pm$0.002}
& 18.87{\scriptsize$\pm$1.50} & 7.03{\scriptsize$\pm$0.42} & 0.153{\scriptsize$\pm$0.002}
& 18.80{\scriptsize$\pm$1.91} & 7.27{\scriptsize$\pm$0.97} & 0.149{\scriptsize$\pm$0.001} \\
Nullu~\cite{yang2025nullu} (Beam)
& 17.60{\scriptsize$\pm$0.87} & 5.90{\scriptsize$\pm$0.36} & 0.153{\scriptsize$\pm$0.002}
& 19.07{\scriptsize$\pm$3.21} & 6.97{\scriptsize$\pm$1.04} & 0.158{\scriptsize$\pm$0.001}
& 17.13{\scriptsize$\pm$1.72} & 6.83{\scriptsize$\pm$1.10} & 0.151{\scriptsize$\pm$0.002} \\
\midrule
\textbf{Ours}
& \textbf{14.93}{\scriptsize$\pm$2.61} & \textbf{4.70}{\scriptsize$\pm$0.66} & 0.151{\scriptsize$\pm$0.005}
& \textbf{16.87}{\scriptsize$\pm$2.81} & \textbf{6.13}{\scriptsize$\pm$1.33} & 0.157{\scriptsize$\pm$0.002}
& \textbf{15.40}{\scriptsize$\pm$1.93} & \textbf{6.07}{\scriptsize$\pm$0.60} & 0.142{\scriptsize$\pm$0.001} \\
\bottomrule
\end{tabular}
\end{adjustbox}

\caption{Comparison of VLMs (LLaVA-1.5, MiniGPT-4, and mPLUG-Owl2) on MSCOCO using strategies for hallucination mitigation. CHAIR$_S$ and CHAIR$_I$ measure sentence- and instance hallucinations (lower is better). All experiments use a maximum token limit of 64. }
\label{tab:coco_results}
\vspace{4pt}
\small
\centering

\setlength{\tabcolsep}{4pt}
\renewcommand{\arraystretch}{1.15}
\begin{adjustbox}{max width=\textwidth}
\begin{tabular}{@{}l@{~}l@{~}|@{~}c@{~}c@{~}c@{~}|@{~}c@{~}c@{~}c@{~}|@{~}c@{~}c@{~}c@{~}|@{~}c@{~}c@{~}c@{~}|@{~}c@{~}c@{~}c@{~}|@{~}c@{~}c@{~}c@{~}|@{~}c@{~}c@{~}c@{~}|@{~}c@{~}c@{~}c@{}}
    \toprule
    \textbf{Setting} & \textbf{Model}
      & \multicolumn{3}{c|}{\textbf{Greedy}}
      & \multicolumn{3}{c|}{\textbf{HALC}}
      & \multicolumn{3}{c|}{\textbf{VCD}}
      & \multicolumn{3}{c|}{\textbf{OPERA}}
      & \multicolumn{3}{c|}{\textbf{Nullu}}
      & \multicolumn{3}{c|}{\textbf{Nullu w/ beam}}
      & \multicolumn{3}{c|}{\textbf{Ours}}
      & \multicolumn{3}{c}{\textbf{Ours w/ beam}} \\
    \cmidrule(lr){3-5}\cmidrule(lr){6-8}\cmidrule(lr){9-11}\cmidrule(lr){12-14}\cmidrule(lr){15-17}\cmidrule(lr){18-20}\cmidrule(lr){21-23}\cmidrule(lr){24-26}
    & & \textbf{Acc} & \textbf{Prec} & \textbf{F$_1$}
      & \textbf{Acc} & \textbf{Prec} & \textbf{F$_1$}
      & \textbf{Acc} & \textbf{Prec} & \textbf{F$_1$}
      & \textbf{Acc} & \textbf{Prec} & \textbf{F$_1$}
      & \textbf{Acc} & \textbf{Prec} & \textbf{F$_1$}
      & \textbf{Acc} & \textbf{Prec} & \textbf{F$_1$}
      & \textbf{Acc} & \textbf{Prec} & \textbf{F$_1$}
      & \textbf{Acc} & \textbf{Prec} & \textbf{F$_1$} \\
    \midrule
    \multirow{3}{*}{\textit{Random}}
      & LLaVA-1.5   & 80.60 & 72.90 & 83.39 & 80.70 & 73.04 & 83.45 & 74.23 & 67.28 & 78.55 & 79.63 & 72.33 & 82.50 & 83.33 & 76.57 & 85.22 & 79.17 & 71.63 & 82.26 & \textbf{85.16} & \textbf{79.28} & \textbf{86.52} & 78.13 & 70.30 & 81.66 \\
      & mPLUG-Owl2  & 83.57 & 76.21 & 85.59 & 83.73 & 76.44 & 85.71 & 78.50 & 71.89 & 81.32 & 87.47 & 82.26 & 88.40 & 78.27 & 70.23 & 81.87 & 85.03 & 78.92 & 86.46 & 82.70 & 75.53 & 84.83 & \textbf{88.20} & \textbf{85.06} & \textbf{88.70} \\
      & MiniGPT4    & 75.57 & 68.57 & 79.44 & 78.83 & 86.38 & 76.39 & 65.87 & 62.38 & 70.08 & 77.17 & 71.24 & 79.96 & 82.40 & 79.49 & 83.23 & 83.17 & 84.48 & 82.84 & \textbf{84.96} & 91.20 & \textbf{83.59} & 82.80 & \textbf{92.48} & 80.85 \\
    \midrule
    \multirow{3}{*}{\textit{Popular}}
      & LLaVA-1.5   & 71.73 & 64.36 & 77.51 & 72.10 & 64.69 & 77.72 & 69.57 & 62.97 & 75.74 & 73.90 & 66.57 & 78.62 & 75.60 & 68.16 & 79.75 & 70.60 & 63.55 & 76.67 & \textbf{78.43} & \textbf{71.30} & \textbf{81.53} & 71.13 & 63.85 & 77.13 \\
      & mPLUG-Owl2  & 75.80 & 67.97 & 80.13 & 75.97 & 68.14 & 80.23 & 74.80 & 68.08 & 78.75 & 79.97 & 72.85 & 82.67 & 72.03 & 64.48 & 77.82 & 79.13 & 71.92 & 82.08 & 77.43 & 69.80 & 81.08 & \textbf{83.40} & \textbf{78.17} & \textbf{84.80} \\
      & MiniGPT4    & 58.17 & 54.74 & 69.29 & 71.30 & 72.58 & 70.46 & 60.27 & 57.37 & 66.80 & 66.50 & 61.05 & 73.12 & 69.50 & 64.37 & 74.12 & 71.30 & 67.76 & 73.90 & \textbf{76.73} & 76.80 & \textbf{76.70} & 76.53 & \textbf{79.56} & 75.26 \\
    \midrule
    \multirow{3}{*}{\textit{Advers.}}
      & LLaVA-1.5   & 65.10 & 59.17 & 73.62 & 65.47 & 59.45 & 73.81 & 64.83 & 59.36 & 72.79 & 69.37 & 62.64 & 75.81 & 68.33 & 61.79 & 75.21 & 66.13 & 60.02 & 74.04 & \textbf{70.86} & \textbf{64.03} & \textbf{76.57} & 65.23 & 59.26 & 73.61 \\
      & mPLUG-Owl2  & 68.27 & 61.51 & 75.46 & 68.47 & 61.68 & 75.57 & 70.03 & 63.73 & 75.63 & 72.67 & 65.55 & 77.75 & 65.87 & 59.64 & 74.19 & 72.37 & 65.27 & 77.58 & 69.73 & 62.84 & 76.18 & \textbf{75.96} & \textbf{69.46} & \textbf{79.40} \\
      & MiniGPT4    & 56.00 & 53.39 & 68.21 & 69.57 & 70.01 & 69.23 & 59.10 & 56.42 & 66.15 & 63.60 & 58.77 & 71.46 & 65.40 & 60.70 & 71.62 & 69.17 & 65.43 & 72.49 & 73.40 & 71.99 & \textbf{74.22} & \textbf{73.76} & \textbf{73.94} & 73.13 \\
    \bottomrule
\end{tabular}
\end{adjustbox}
\caption{POPE evaluation under Random, Popular, and Adversarial settings. We report Accuracy (Acc), Precision (Prec), and F$_1$ score. Higher is better.}
\label{tab:pope}

\end{table}
\myparagraph{Implementation Details}
\label{sec:implementation_details}
All experiments are conducted using the official model implementations, with the proposed Wiener edit applied \emph{only} to the \texttt{down\_proj} matrices in the FFN blocks of selected LLM layers.  Hyperparameters are determined using a small held-out validation set of 100 randomly sampled MSCOCO images (disjoint from the evaluation set). We select the intervention layer range $\mathcal{L}$ and the sharpness parameter $\alpha$ based on performance on this validation subset, and fix the resulting configuration across all reported experiments. This procedure yields consistent, model-specific configurations. For LLaVA-1.5, we use layers $20$--$32$ with $\alpha=60.0$. For mPLUG-Owl2, we use layers $20$--$32$ with $\alpha=20.0$. For MiniGPT-4, we use layers $24$--$32$ with $\alpha=10.0$. These configurations are applied uniformly across CHAIR, POPE, and MME evaluations without further tuning.

For generation, we use beam search with beam size $3$ on CHAIR, and greedy decoding (beam size $1$) for MME. Unless otherwise specified, the maximum generation length is set to 64 tokens on CHAIR and POPE, 128 tokens for MME. We use the official evaluation scripts from HALC~\cite{chen2024halc} for all CHAIR-based evaluations.

\subsection{Quantitative Results}

\myparagraph{CHAIR.}
Table~\ref{tab:coco_results} reports caption hallucination and caption quality on CHAIR across three VLM backbones. 
Our \emph{Wiener} weight edit consistently achieves the lowest hallucination rates across all evaluated models. 
Specifically, it reduces CHAIR$_S$ to \textbf{14.93}, \textbf{16.87}, and \textbf{15.40} for LLaVA-1.5, MiniGPT-4, and mPLUG-Owl2, respectively, outperforming prior approaches such as HALC and OPERA. 
Similarly, instance-level hallucination (CHAIR$_I$) is reduced to \textbf{4.70}, \textbf{6.13}, and \textbf{6.07}, yielding the best performance across all architectures. 
While some decoding-based methods achieve slightly higher BLEU scores, our approach maintains competitive caption quality while consistently minimizing hallucinated objects. 

\myparagraph{POPE.}
Table~\ref{tab:pope} evaluates object-presence grounding under the Random, Popular, and Adversarial POPE settings. 
Our method consistently improves factual grounding across the evaluated architectures and splits. 
For LLaVA-1.5, our approach achieves the strongest results in most settings, obtaining the highest accuracy and F$_1$ scores in the Random, Popular, and Adversarial splits compared to both decoding-time mitigation strategies (HALC, VCD, OPERA) and representation-editing approaches such as Nullu. 
MiniGPT-4 exhibits similar trends, where our method significantly improves accuracy and F$_1$, while also achieving particularly high precision, indicating fewer false-positive object predictions. 
For mPLUG-Owl2, the beam-search variant of our method yields the best performance across the three splits, outperforming competing approaches in both accuracy and F$_1$.


\myparagraph{MME.}
We evaluate our method on the 10 relevant subsets of the MME benchmark, including OCR, count, scene, artwork, color, existence, position, celebrity, landmark, and posters (code reasoning, numerical calculations, text translation, and commonsense reasoning are irrelevant). Figure~\ref{fig:mme} reports Dolan--Mor\'e profiles, which plot the fraction of tasks for which a method achieves performance within a ratio $\tau$ of the best method, providing a holistic comparison across diverse subsets. Our approach consistently achieves the strongest profile on both mPLUG-Owl2 and LLaVA-1.5, dominating competing methods over a wide range of $\tau$.
These results demonstrate that Wiener-based spectral shaping improves general visual grounding beyond object hallucination benchmarks. Per-subset MME scores for LLaVA-1.5 are reported in Appendix~\ref{app:mme_subsets}.

{Additional experiments evaluating the method on Gemma3-4B-it (Appendix~\ref{app:gemma3}), per-subset MME scores (Appendix~\ref{app:mme_subsets}), the TempCompass video understanding benchmark (Appendix~\ref{app:tempcompass}), cross-domain transfer (Appendix~\ref{app:cross_domain}), and a faithfulness-versus-informativeness analysis (Appendix~\ref{app:informativeness}) are provided in the supplementary material.}

\myparagraph{FaithDial.}
Table~\ref{tab:faithdial_full} reports the results on FaithDial~\citep{dziri2022faithdial}. Applying the Wiener-based representation filtering to Dream-v0-Instruct-7B\footnote{\url{https://huggingface.co/Dream-org/Dream-v0-Instruct-7B}}~\citep{ye2508dream} substantially improves both grounding quality and hallucination rates. On overlap-based metrics, the filtered model achieves consistent gains, improving BLEU by +2.21 and ROUGE-1/2/L by +2.81, +3.07, and +3.36, respectively. Grounding-oriented metrics also show large improvements, as can be seen, Token-F1 increases by 9.86 points and Q$^2$ by 10.56, indicating stronger alignment between generated responses and the provided knowledge. Most importantly, hallucination rates drop significantly. FaithCritic reports a reduction from 92\% to 74\%, while the Claude Sonnet 4.6 judge shows an even larger decrease from 93\% to 50\%, corresponding to a 43-point absolute reduction in unfaithfulness rate. These results suggest that the proposed representation filtering not only reduces hallucinated statements in diffusion-based language generation but also improves the factual grounding and content fidelity of the produced responses.

\begin{table*}[h]
\centering
\small
\begin{tabular}{@{}lccccccc@{}}
\toprule
\textbf{Model} & \textbf{BLEU} & \textbf{ROUGE-1} & \textbf{ROUGE-2} & \textbf{ROUGE-L} & \textbf{Token-F1} & \textbf{Q$^2$} & \textbf{FaithCritic$\downarrow$} \\
\midrule
Dream          & 3.89 & 25.32 & 7.28 & 18.97 & 29.00 & 19.19 & 85.92 \\
With Nullu     & 3.95 & 25.95 & 7.83 & 19.53 & \textbf{35.00} & 24.63 & 82.91 \\
With Wiener    & \textbf{4.47} & \textbf{26.22} & \textbf{8.16} & \textbf{20.45} & 33.30 & \textbf{26.24} & \textbf{75.39} \\
\bottomrule
\end{tabular}
\caption{Full FaithDial test set evaluation on Dream-v0-Instruct-7B.}
\label{tab:faithdial_full}
\end{table*}

\begin{figure*}[t]

\centering
\begin{tabular}{ccc}
\includegraphics[width=0.48\textwidth]{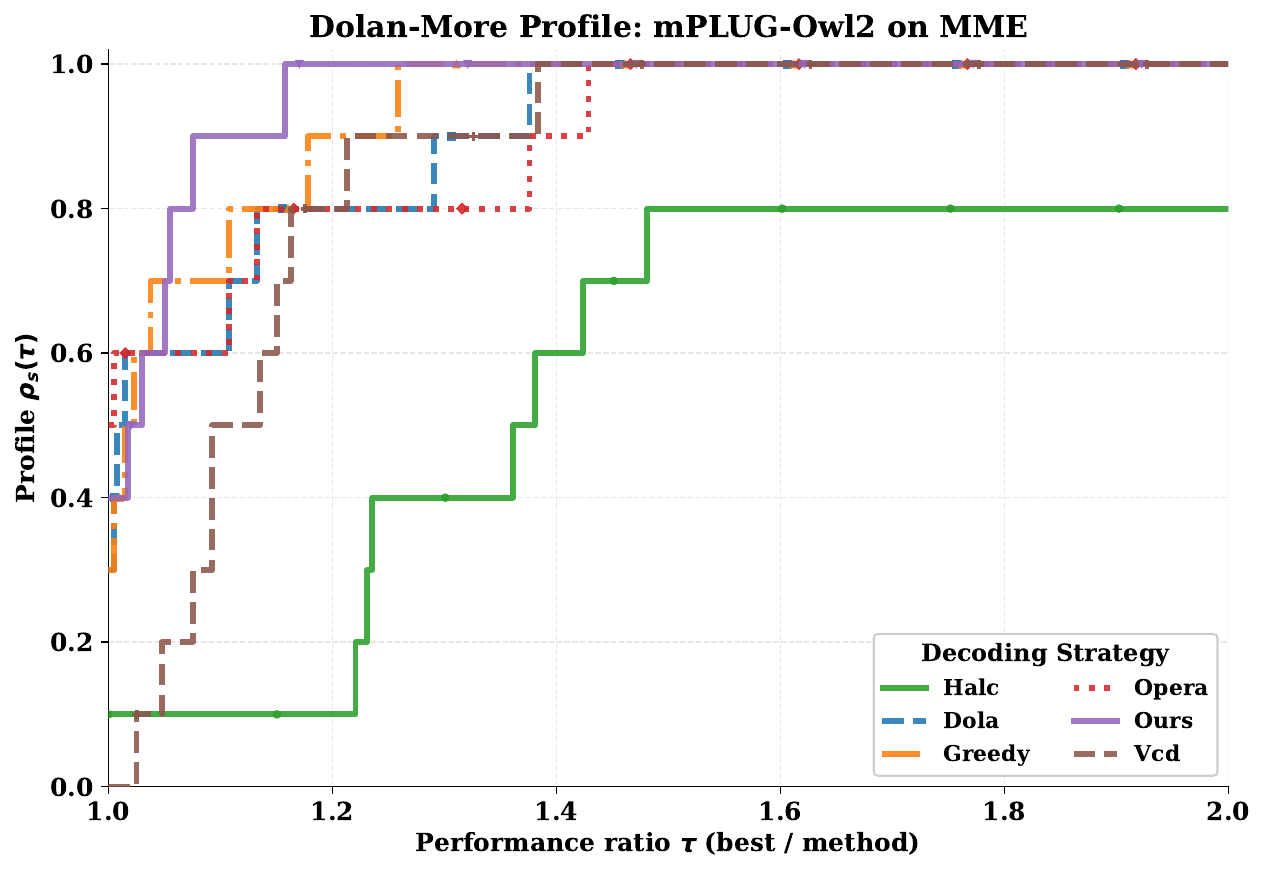} &
\includegraphics[width=0.48\textwidth]{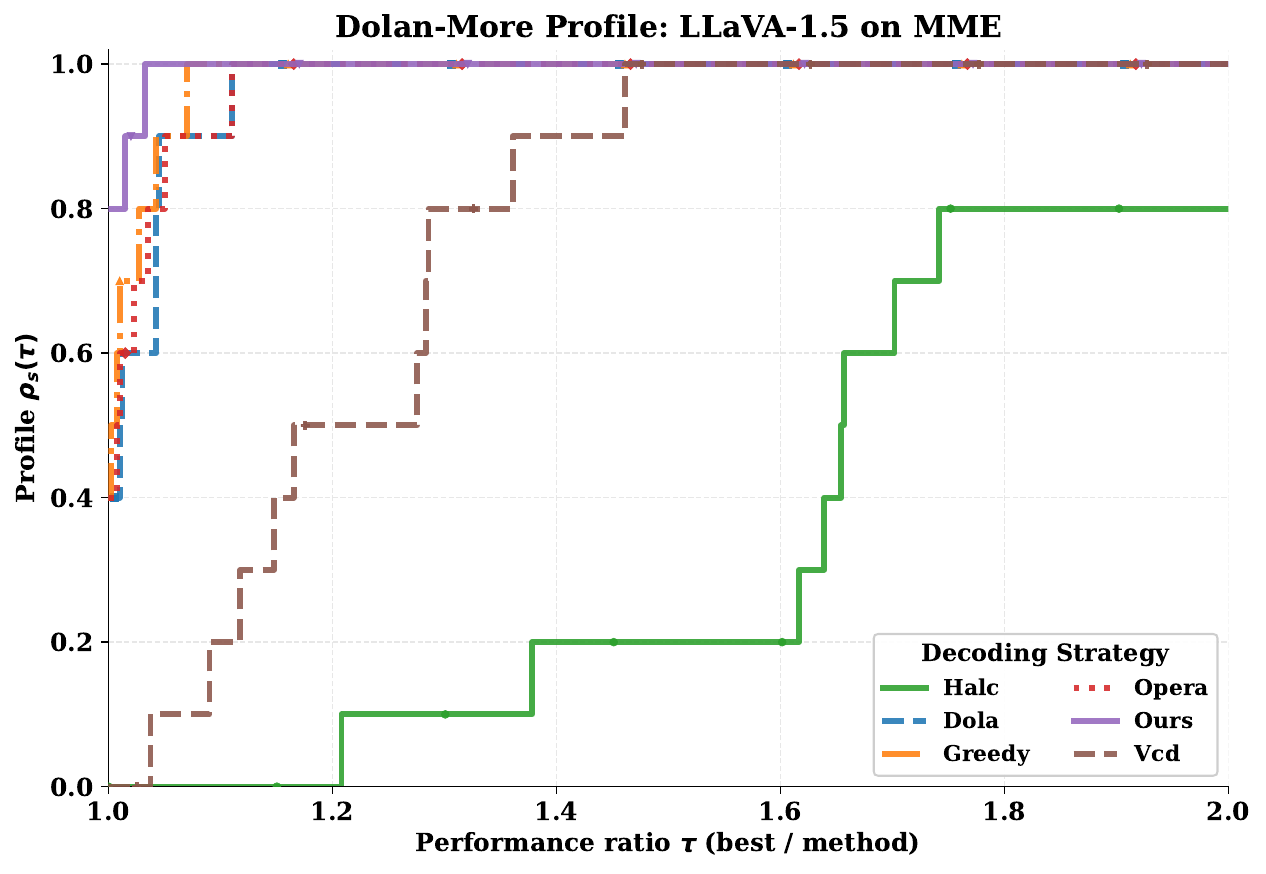} &
\end{tabular}
\caption{Dolan--Mor\'e performance profiles on 10 MME subsets for mPLUG-Owl2 (left) and LLaVA-1.5 (right).}
\label{fig:mme}
\end{figure*}
\subsection{Ablation Studies}
\label{sec:ablations}
We analyze which components of the proposed Wiener edit contribute to the observed improvements by performing controlled ablations on the CHAIR benchmark using MiniGPT-4. Each ablation isolates a specific element of the representation correction while keeping the intervention layers and evaluation protocol identical. Results are reported in Table~\ref{tab:ablations_Wiener}.
\begin{table}[t]
\centering
\small
\resizebox{\textwidth}{!}{%
\begin{tabular}{@{}lcccccc@{}}
\toprule
\textbf{Metric} & \textbf{Baseline} & \textbf{Mean distortion subtraction} & \textbf{Uniform shrinkage} & \textbf{layers 4--12} & \textbf{layers 14--24} & \textbf{Our (layers 24--32)} \\
\midrule
\textbf{CHAIR$_S\downarrow$} & 23.0 & 22.0 & 24.0 & 22.0 & 14.0 & \textbf{13.0} \\
\textbf{CHAIR$_I\downarrow$} & 8.4  & 8.2  & 8.5  & 7.0  & 5.0  & \textbf{4.9}  \\
\textbf{BLEU$\uparrow$}      & 0.157 & 0.157 & 0.157 & 0.186 & 0.166 & 0.162 \\
\bottomrule
\end{tabular}%
}
\caption{Ablation analysis on CHAIR using MiniGPT-4. Lower CHAIR$_S$ and CHAIR$_I$ indicate fewer hallucinations, while higher BLEU indicates better caption quality.}
\label{tab:ablations_Wiener}
\end{table}

\myparagraph{Mean distortion subtraction.}
The first ablation removes only the empirical mean distortion 
$\mu_d = \frac{1}{N}\sum_i d_i$ from the MLP output, without applying the Wiener operator. 
This corresponds to a first-order correction that compensates for a global bias between truthful and hallucinated representations. 
While this adjustment slightly reduces hallucination metrics, the improvement remains modest. 
This behavior indicates that hallucination is not primarily caused by a constant offset in representation space. 
Instead, hallucinated activations exhibit structured variation that cannot be captured by a single mean vector. 
Consequently, subtracting the mean distortion removes only a small portion of the error while leaving the dominant hallucination directions largely unchanged.

\myparagraph{Uniform shrinkage.}
In the second ablation we replace the covariance-driven operator with an isotropic scaling 
$P = cI$, applying the same editing pipeline but without exploiting the covariance structure of the distortions. 
This operation uniformly contracts the representation space regardless of the direction of variation. 
Empirically, this naive shrinkage slightly degrades performance relative to the baseline. 
The degradation suggests that hallucination-related variance is highly anisotropic: certain directions are strongly associated with hallucinations, while others encode useful semantic or visual information. 
Uniform scaling therefore suppresses informative directions together with hallucination-dominated ones, leading to a loss of useful signal

\myparagraph{Layer-wise Wiener editing.}
We analyze the effect of applying the Wiener edit across different layer ranges. As shown in Table~\ref{tab:ablations_Wiener}, editing early layers (4--12) yields only marginal improvements, indicating that low-level representations are less directly associated with hallucination structure. In contrast, intervening in mid-level layers (14--24) leads to a substantial reduction in both CHAIR$_S$ and CHAIR$_I$, suggesting that hallucination-related distortions emerge at semantic abstraction stages. The strongest gains are obtained when editing deeper layers (24--32), where the method achieves the lowest hallucination rates, confirming that distortion is most pronounced near the output stage. These results demonstrate that Wiener-based spectral shaping is most effective when applied to high-level representations, where it aligns generation with visual evidence while preserving fluency. 
Additional ablations of the Wiener Sharpness parameter \texorpdfstring{$\alpha$}{alpha} could be found in Appendix.~\ref{app:sensitivity_alpha}.

\myparagraph{Assumption validation.} We empirically verify the weak cross-covariance assumption underlying our estimator: image-aligned truthfu
l/hallucinated pairs yield low covariance-additivity error and high spectral concentration, both of which collapse under a random-pair control
(Appendix~\ref{sec:assumption_validation}).

\section{Conclusion}

This work addresses object hallucinations in vision-language models by analyzing the covariance structure of hidden representations and introducing a Wiener-based representation filtering framework. Hallucination mitigation is cast as a linear estimation problem in representation space, yielding a {training-free, post-hoc} correction that attenuates hallucination-dominated spectral modes from a lightweight one-time calibration (forward passes and empirical second-order statistics) with no gradient updates. The correction requires no architectural changes and is absorbed directly into existing weight matrices without increasing inference cost, while consistently reducing object hallucination across multiple VLMs and preserving caption quality and fluency.

\section*{Acknowlegments}
This work was supported by a Tel Aviv University Center for AI and Data Science (TAD) grant
and by Len Blavatnik and the Blavatnik Family
foundation. This research was also supported by
the Ministry of Innovation, Science \& Technology,
Israel (1001576154) and the Michael J. Fox Foundation (MJFF-022407). The contribution of Ameen Ali is part of a  PhD thesis research conducted at Tel Aviv University.

\bibliography{colm2026_conference}
\bibliographystyle{colm2026_conference}
\newpage
\appendix

\section{Full Algorithm}
\label{app:algorithm}

Algorithm~\ref{alg:wiener_alg} summarizes the proposed Wiener editing procedure. Given paired hallucinated and truthful representations at each selected layer $\ell \in \mathcal{L}$, we estimate the truthful covariance $\Sigma_T$ and the distortion covariance $\Sigma_H$, compute a Wiener-shaped spectral operator $F_\alpha$ from their induced distortion-to-signal ratios, and absorb this operator into the FFN output projection:
\[
\widetilde{W}^{\mathrm{out}}_\ell = F_\alpha W^{\mathrm{out}}_\ell.
\]
Since the edit is computed offline and merged directly into the existing weights, the architecture and inference-time cost remain unchanged.

\begin{algorithm*}[h]
\caption{Wiener Spectral Shaping for Post-hoc Editing}
\label{alg:wiener_alg}
\small
\begin{algorithmic}[1]
\Require Calibration set $D=\{(I_i,c_i^+,c_i^-)\}_{i=1}^N$, selected layers $\mathcal{L}$, sharpness $\alpha$
\Ensure Edited weights $\{\widetilde{W}^{\mathrm{out}}_\ell\}_{\ell\in\mathcal{L}}$

\For{each layer $\ell \in \mathcal{L}$}
    \State Extract paired features $x_i^+ \gets \mathrm{Feat}_\ell(I_i,c_i^+)$, \, $x_i^- \gets \mathrm{Feat}_\ell(I_i,c_i^-)$
    \State Form distortions $d_i \gets x_i^+ - x_i^-$ and truthful mean $\mu_T \gets \frac{1}{N}\sum_{i=1}^N x_i^-$
    \State $\Sigma_T \gets \frac{1}{N}\sum_{i=1}^N (x_i^- - \mu_T)(x_i^- - \mu_T)^\top$
    \State $\Sigma_H \gets \frac{1}{N}\sum_{i=1}^N d_i d_i^\top$
    \State Eigendecompose $\Sigma_H = Q\Lambda Q^\top$
    \For{each mode $j=1,\dots,d$}
        \State $\tau_j^2 \gets q_j^\top \Sigma_T q_j$, \quad $\nu_j \gets \lambda_j/\tau_j^2$, \quad $\tilde{\gamma}_j \gets (1+\nu_j)^{-\alpha}$
    \EndFor
    \State $F_\alpha \gets Q\,\mathrm{diag}(\tilde{\gamma}_1,\dots,\tilde{\gamma}_d)\,Q^\top$
    \State $\widetilde{W}^{\mathrm{out}}_\ell \gets F_\alpha W^{\mathrm{out}}_\ell$
\EndFor
\State Replace $W^{\mathrm{out}}_\ell$ with $\widetilde{W}^{\mathrm{out}}_\ell$ for all $\ell \in \mathcal{L}$
\end{algorithmic}
\end{algorithm*}

\section{Proof of Lemma~\ref{lem:spectral_stability}}
\label{app:proof}

\begin{proof}
Define the line segment
\begin{equation}
\mathbf{C}_t := \mathbf{C} + t(\widehat{\mathbf{C}} - \mathbf{C}) \quad \text{for } t\in[0,1].
\end{equation}
Since $I$ is convex and contains $\sigma(\mathbf{C}) \cup \sigma(\widehat{\mathbf{C}})$, we have $\sigma(\mathbf{C}_t)\subset I$ for all $t\in[0,1]$. Let
\begin{equation}
\Delta := \widehat{\mathbf{C}} - \mathbf{C}.
\end{equation}

Consider the matrix-valued function
\begin{equation}
\Phi(t) := g(\mathbf{C}_t).
\end{equation}
By standard results on spectral functions (Daleckii--Kre\u{\i}n theorem), $\Phi$ is differentiable and its derivative is the Fr\'echet derivative of $g$ at $\mathbf{C}_t$ applied to $\Delta$:
\begin{equation}
\Phi'(t) = Dg(\mathbf{C}_t)[\Delta].
\end{equation}

Therefore, by the fundamental theorem of calculus,
\begin{equation}
g(\widehat{\mathbf{C}}) - g(\mathbf{C}) = \Phi(1)-\Phi(0) = \int_0^1 Dg(\mathbf{C}_t)[\Delta] \, dt.
\end{equation}

Taking the operator norm and using the triangle inequality gives
\begin{equation}
\|g(\widehat{\mathbf{C}}) - g(\mathbf{C})\|_2
\le \int_0^1 \|Dg(\mathbf{C}_t)[\Delta]\|_2 \, dt
\le \left(\sup_{t\in[0,1]} \|Dg(\mathbf{C}_t)\|_{\mathrm{op}}\right)\,\|\Delta\|_2,
\end{equation}
where
\begin{equation}
\|Dg(\mathbf{C}_t)\|_{\mathrm{op}} := \sup_{\|X\|_2=1} \|Dg(\mathbf{C}_t)[X]\|_2
\end{equation}
is the induced operator norm of the Fr\'echet derivative.

It remains to bound $\|Dg(\mathbf{C}_t)\|_{\mathrm{op}}$ uniformly by the operator-Lipschitz constant $K$. Fix any symmetric $\mathbf{A}$ with $\sigma(\mathbf{A})\subset I$ and any symmetric direction $\mathbf{X}$. For $\varepsilon \neq 0$ small enough, $\sigma(\mathbf{A}+\varepsilon \mathbf{X}) \subset I$, and by the operator-Lipschitz assumption,
\begin{equation}
\|g(\mathbf{A}+\varepsilon \mathbf{X})-g(\mathbf{A})\|_2 \le K\,\|\varepsilon \mathbf{X}\|_2 = K |\varepsilon|\,\|\mathbf{X}\|_2.
\end{equation}

Divide by $|\varepsilon|$ and take $\varepsilon \to 0$. By the definition of the Fr\'echet derivative,
\begin{equation}
\|Dg(\mathbf{A})[\mathbf{X}]\|_2 = \lim_{\varepsilon \to 0} \left\|\frac{g(\mathbf{A}+\varepsilon \mathbf{X}) - g(\mathbf{A})}{\varepsilon}\right\|_2 \le K \|\mathbf{X}\|_2.
\end{equation}

Taking the supremum over $\|\mathbf{X}\|_2=1$ yields $\|Dg(\mathbf{A})\|_{\mathrm{op}} \le K$. Applying this with $\mathbf{A}=\mathbf{C}_t$ for all $t\in[0,1]$ gives
\begin{equation}
\sup_{t\in[0,1]} \|Dg(\mathbf{C}_t)\|_{\mathrm{op}} \le K.
\end{equation}

Substituting into the earlier bound finishes the proof:
\begin{equation}
\|g(\widehat{\mathbf{C}}) - g(\mathbf{C})\|_2 \le K\, \|\widehat{\mathbf{C}} - \mathbf{C}\|_2.
\end{equation}

\end{proof}

\section{Covariance Estimation Details}
\label{app:covariance}
We use empirical covariance estimates with minimal numerical stabilization. The hallucination covariance $\boldsymbol{\Sigma}_H$ is computed as the sample second moment of the paired hallucination-truth residuals $\boldsymbol{\Delta}_i = \mathbf{h}_i^+ - \mathbf{h}_i^-$, while $\boldsymbol{\Sigma}_T$ is computed as the sample covariance of mean-centered truthful activations. The main editor does not use Ledoit-Wolf shrinkage, rank truncation, or an explicit ridge term on $\boldsymbol{\Sigma}_H$ or $\boldsymbol{\Sigma}_T$.

The stabilizing steps are numerical rather than statistical regularization: we clamp tiny negative estimates of $\tau_j^2$ to zero, guard near-zero denominators in the Wiener gain by setting the corresponding gain to 1, and compute the eigendecomposition and edit matrix in float32 before casting back to the model dtype. A small ridge term ($\epsilon = 10^{-20}$) is used only in optional diagnostic scripts that explicitly solve inverse systems for analysis, it is not used in the main weight edit.

In our experiments, we estimate the covariance statistics using the LURE dataset with approximately 3k paired truthful/hallucinated samples per model. Empirically, stable performance is achieved once the dominant hallucination covariance structure is sufficiently captured, consistent with the heavy-tailed low-rank spectrum shown in Appendix~\ref{app:spectrum}.

\newpage
\section{Hallucination Spectrum}
\label{app:spectrum}
\myparagraph{Hallucination covariance spectrum.}
Figure~\ref{fig:spectra} visualizes the eigenvalue spectrum of the hallucination covariance matrix 
$\Sigma_H = Q \Lambda Q^{\top}$ for three representative vision--language models: 
LLaVA-7B, MiniGPT-4, and mPLUG-Owl2. 
Across all architectures, the spectrum exhibits a characteristic heavy-tailed structure, 
with a small number of dominant eigenvalues followed by a long, smoothly decaying tail. 
The leading eigenvalues correspond to high-variance directions in representation space 
that capture the majority of hallucination-related variation. 
In contrast, the remaining eigenmodes contribute progressively smaller variance and 
behave similarly to noise-like fluctuations.

This spectral structure suggests that hallucination behavior is largely concentrated 
within a low-dimensional subspace of the model representation. 
Despite architectural differences between the evaluated VLMs, the qualitative shape of the 
spectrum remains consistent, indicating that this phenomenon is not model-specific but 
rather a general property of multimodal language models. 
Motivated by this spectral structure, we view hallucination as a form of structured distortion distributed across latent modes. This perspective suggests operating directly in the spectral domain, where each mode can be treated according to its distortion-to-signal characteristics. We therefore adopt a Wiener-inspired formulation, in which the representation is shaped through a covariance-driven spectral operator that balances fidelity to the underlying signal with robustness to distortion. This yields a principled mechanism for preserving semantic consistency and fluency while improving alignment with the visual input.
\begin{figure*}[h]
\centering
\includegraphics[width=0.995\textwidth]{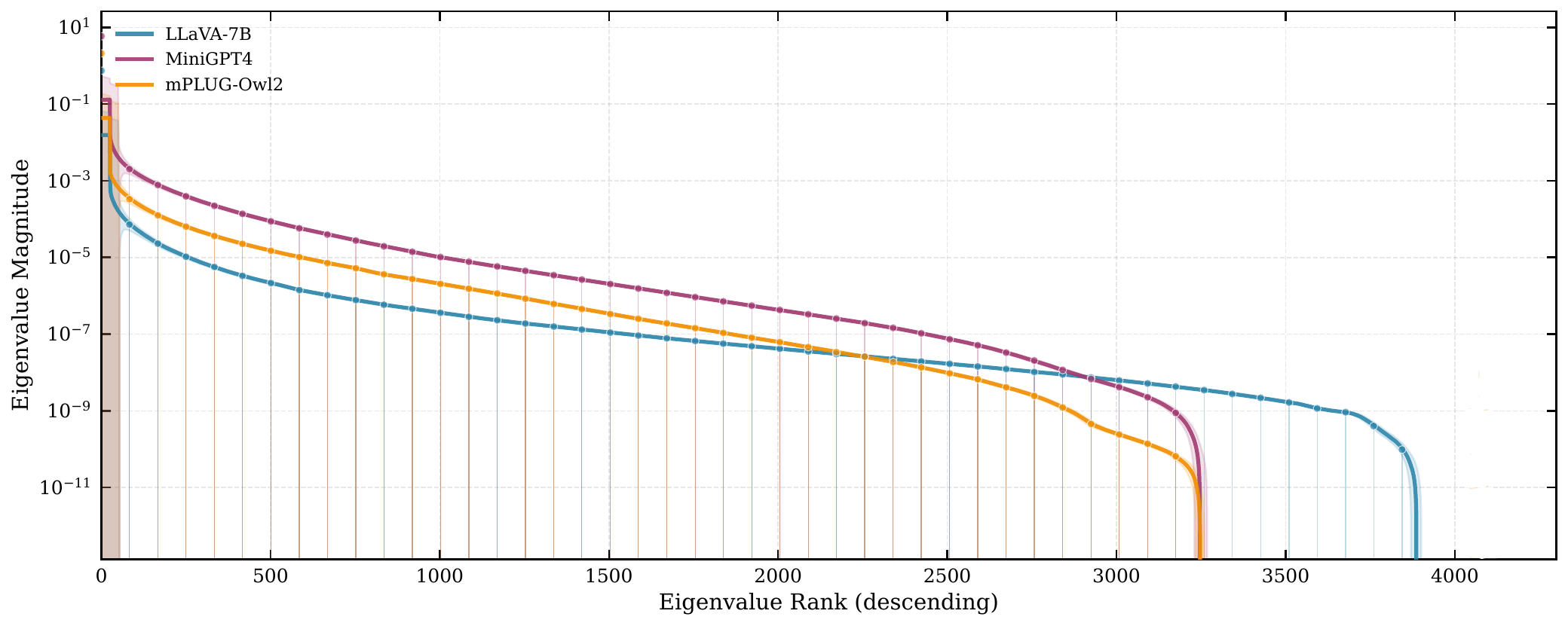}
 \vspace{-4pt}
 \caption{Hallucination spectrum (eigenvalues of $\Sigma_H = Q\Lambda Q^{\top}$) for three 
vision–language models.  
The curves show a small set of high-variance “spikes’’ followed by a long decaying tail, 
indicating that hallucination behavior is concentrated in a low-dimensional set of dominant 
eigenmodes, with the remaining spectrum reflecting noise-like variation.}\label{fig:spectra}
\vspace{-7pt}
\end{figure*}
\newpage
\section{Sensitivity of the Wiener Sharpness Parameter \texorpdfstring{$\alpha$}{alpha}}
\label{app:sensitivity_alpha}

We analyze the sensitivity of the proposed Wiener edit to the
sharpness parameter $\alpha$. In our formulation, $\alpha$ does not determine
which spectral modes are emphasized or de-emphasized, those ratios are
determined by the Wiener gains derived from the signal and distortion
covariances. Instead, $\alpha$ controls the \emph{sharpness} with which these
mode-wise gains are expressed, effectively adjusting how strongly the Wiener
shaping is applied.

To evaluate robustness, we vary $\alpha$ over a wide range while keeping all
other components fixed. Figure~\ref{fig:sensitivity} reports results on 100
randomly sampled MSCOCO images using LLaVA-1.5. Panels (a) and (b) show the
resulting CHAIR$_S$ and CHAIR$_I$ metrics, while panel (c) reports BLEU score.

Across the explored range, the method exhibits stable behavior: hallucination
metrics improve consistently for moderate and large $\alpha$, while BLEU remains
largely unchanged. Importantly, performance does not depend on a narrowly tuned
value, rather, a broad interval of $\alpha$ yields similar improvements. This
behavior reflects the Wiener formulation itself: the relative weighting of
spectral modes is governed by the covariance geometry, while $\alpha$ simply
modulates the sharpness of this transformation.

Overall, the results indicate that the method is robust to the choice of
$\alpha$, with strong performance across a wide range of values. In practice,
we find that values in the range $\alpha \in [10,80]$ provide a good balance
between improved grounding (lower CHAIR metrics) and stable caption quality.

\begin{figure*}[h]
\centering
\begin{tabular}{ccc}
\includegraphics[width=0.31\textwidth]{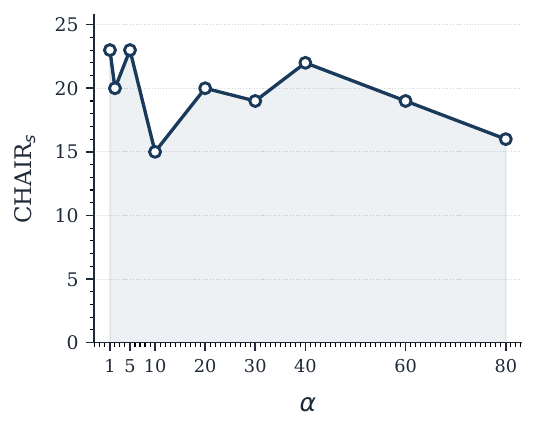} &
\includegraphics[width=0.31\textwidth]{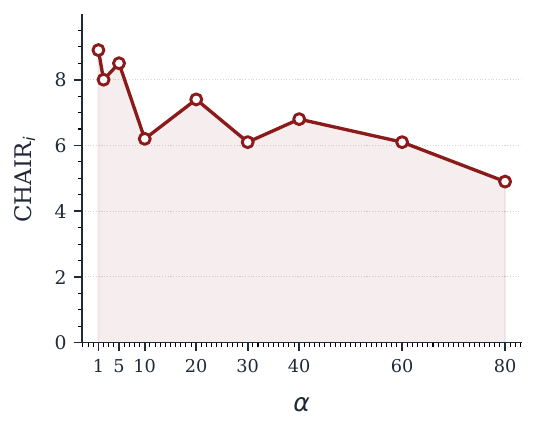} &
\includegraphics[width=0.31\textwidth]{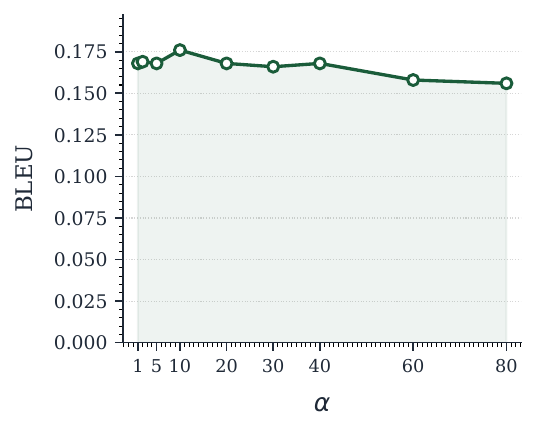} \\
(a) CHAIR$_S$ & (b) CHAIR$_I$ & (c) BLEU
\end{tabular}
\caption{
Sensitivity of the Wiener sharpness parameter $\alpha$ on LLaVA-1.5 evaluated on
100 MSCOCO validation images. Lower CHAIR values indicate fewer hallucinations,
while higher BLEU indicates better caption quality.
}

\label{fig:sensitivity}
\end{figure*}
\newpage
\section{MME Per-Subset Scores}
\label{app:mme_subsets}
Table~\ref{tab:mme_subsets} reports the actual per-subset MME scores for LLaVA-1.5 across all evaluated methods. Our approach achieves the highest mean score of 148.5, outperforming all competing methods including Greedy (147.0) and OPERA (145.6). Notable improvements are observed in Count (+10 over Greedy), Artwork (+5), Celebrity (+3), and Posters (+2), while maintaining competitive or identical performance on tasks such as OCR, Existence, and Position.

\begin{table*}[h]
\centering
\small
\setlength{\tabcolsep}{4pt}
\begin{tabular}{@{}lcccccccccc@{\hspace{10pt}}c@{}}
\toprule
\textbf{Method} & \textbf{OCR} & \textbf{Count} & \textbf{Scene} & \textbf{Artwork} & \textbf{Color} & \textbf{Exist.} & \textbf{Pos.} & \textbf{Celeb.} & \textbf{Land.} & \textbf{Poster} & \textbf{Mean} \\
\midrule
DoLa   & 132 & 138 & 155 & 125 & \textbf{170} & \textbf{188} & \textbf{128} & 130 & 150 & 138 & 145.4 \\
Greedy & 132 & 143 & 155 & 125 & \textbf{170} & \textbf{188} & 130 & 132 & \textbf{152} & 143 & 147.0 \\
Halc   & 80  & 75  & 112 & 107 & 100 & 115 & 78  & 76  & 95  & 72  & 91.0  \\
Opera  & 132 & 138 & 155 & 125 & \textbf{170} & \textbf{188} & \textbf{128} & 130 & 150 & 140 & 145.6 \\
VCD    & 98  & 105 & 145 & 110 & 132 & 170 & 100 & 130 & 130 & 110 & 123.0 \\
\midrule
Ours   & 132 & \textbf{153} & \textbf{157} & \textbf{130} & 165 & \textbf{188} & \textbf{128} & \textbf{135} & \textbf{152} & \textbf{145} & \textbf{148.5} \\
\bottomrule
\end{tabular}
\caption{MME per-subset scores for LLaVA-1.5 across all evaluated methods.}
\label{tab:mme_subsets}
\end{table*}

{
\section{Generalization to Stronger Models: Gemma3}
\label{app:gemma3}
To demonstrate effectiveness on more recent and stronger VLM backbones, we evaluate the proposed Wiener filtering on Gemma3-4B-it, which was released well after the models in our original evaluation. For this experiment, our method edits layers 12--26 with $\alpha = 10.0$. For Nullu, we report its best validation configuration with top\_k $= 1$. Results are reported in Table~\ref{tab:gemma3_chair}.
}
The proposed Wiener filter reduces CHAIR$_S$ from 0.357 to 0.261 and CHAIR$_I$ from 0.080 to 0.073, outperforming Nullu across all hallucination metrics while maintaining competitive caption quality (BLEU). This corresponds to a relative reduction of approximately 27\% in sentence-level hallucination rate over the baseline, and roughly 16\% over Nullu.

\begin{table}[h]
\centering
\small
\begin{tabular}{@{}lccccc@{}}
\toprule
\textbf{Model} & \textbf{CHAIR$_S$$\downarrow$} & \textbf{CHAIR$_I$$\downarrow$} & \textbf{BLEU$\uparrow$} & \textbf{Avg obj.} & \textbf{Avg len.} \\
\midrule
Gemma3          & 0.357 & 0.080 & \textbf{0.040} & 8.21 & 184.2 \\
Gemma3 + Nullu  & 0.311 & 0.077 & 0.038 & 7.37 & 180.1 \\
Gemma3 + Wiener & \textbf{0.261} & \textbf{0.073} & 0.037 & 7.11 & 180.2 \\
\bottomrule
\end{tabular}
\caption{CHAIR evaluation on Gemma3-4B-it. The Wiener filter achieves the lowest hallucination rates while preserving caption quality.}
\label{tab:gemma3_chair}
\end{table}

\section{Faithfulness versus Informativeness}
\label{app:informativeness}
An important question is whether hallucination reduction arises from improved visual grounding or simply from more conservative generation. As shown in Table~\ref{tab:gemma3_chair}, the baseline Gemma3 model produces an average of 8.21 object mentions and 184.2 generated tokens per caption. After applying Wiener filtering, the average object count decreases to 7.11 and generation length to 180.2 tokens. While both editing methods slightly reduce the number of mentioned objects, the reduction is relatively small (approximately 13\% in object count and 2\% in generation length), whereas the improvement in hallucination metrics is substantially larger. These results suggest that the method improves the precision of object mentions by selectively suppressing hallucination-dominated representation directions while largely preserving overall generation length and descriptive content.

{
\section{Video Understanding: TempCompass}
\label{app:tempcompass}
To evaluate whether the proposed representation editing framework generalizes beyond image-based hallucination, we tested on TempCompass~\citep{liu2024tempcompass}, a challenging benchmark that measures temporal reasoning through Caption Matching (CM), Captioning (Cap), Multiple Choice QA (MCQ), and Yes/No QA (Y/N) across Action, Attribute Change, Direction, Order, and Speed dimensions. Despite being calibrated using image-based hallucination supervision, the proposed method consistently improves performance across nearly all TempCompass tasks and dimensions (Table~\ref{tab:tempcompass}). Average gains are observed from 54.2$\to$58.5 on CM, 14.5$\to$23.4 on Captioning, 37.1$\to$38.9 on MCQ, and 18.4$\to$24.0 on Y/N QA. These results suggest that the learned spectral structure captures general multimodal representation distortions that extend to temporal video reasoning. For this evaluation, we used editing layers 12--34 with $\alpha = 40.0$, without task-specific retuning.
}
\begin{table*}[h]
\centering
\small
\setlength{\tabcolsep}{5pt}
\begin{tabular}{@{}l cc cc cc cc@{}}
\toprule
& \multicolumn{2}{c}{\textbf{CM}} & \multicolumn{2}{c}{\textbf{Caption}}
& \multicolumn{2}{c}{\textbf{MCQ}} & \multicolumn{2}{c@{}}{\textbf{Yes/No}} \\
\cmidrule(lr){2-3} \cmidrule(lr){4-5} \cmidrule(lr){6-7} \cmidrule(l){8-9}
\textbf{Dimension} & Base & Ours & Base & Ours & Base & Ours & Base & Ours \\
\midrule
Action       & 58.6 & \textbf{63.6} & 14.1 & \textbf{21.7} & 46.4 & \textbf{53.6} & 14.0 & \textbf{17.6} \\
Attr.\ Change & 56.9 & \textbf{59.7} & 10.9 & \textbf{22.4} & 34.7 & \textbf{36.1} & 27.5 & \textbf{33.0} \\
Direction    & 48.3 & \textbf{54.1} & 17.0 & \textbf{26.8} & 33.4 & \textbf{34.9} & 16.2 & \textbf{21.7} \\
Order        & 56.3 & \textbf{60.0} & 22.5 & \textbf{31.5} & \textbf{37.1} & 31.8 & 25.9 & \textbf{37.4} \\
Speed        & 51.5 & \textbf{58.5} & 7.5  & \textbf{13.9} & 33.1 & \textbf{36.9} & 12.3 & \textbf{16.2} \\
\midrule
\textbf{Avg} & 54.2 & \textbf{58.5} & 14.5 & \textbf{23.4} & 37.1 & \textbf{38.9} & 18.4 & \textbf{24.0} \\
\bottomrule
\end{tabular}
\caption{TempCompass evaluation on Gemma3-4B-it. The proposed method improves performance across nearly all temporal reasoning dimensions and task formats despite being calibrated on image-based hallucination data.}
\label{tab:tempcompass}
\end{table*}
{
\section{Cross-Domain Transfer}
\label{app:cross_domain}
To evaluate whether the learned hallucination covariance transfers across domains without retuning, we learned the covariance structure from one domain and evaluated on another. Specifically, we learned from LURE and evaluated on TruthfulQA~\citep{lin2022truthfulqa}, and conversely learned from HaluEval~\citep{li2023halueval} and evaluated on CHAIR. In both directions, the method reduced visual hallucinations and improved textual truthfulness (Table~\ref{tab:cross_domain}), suggesting that the learned spectral structure captures hallucination-related model biases that are not limited to a single dataset or modality. All experiments were conducted on Gemma3-4B-it with layers 12--34 and $\alpha = 40.0$.
}
\begin{table}[h]
\centering
\small
\begin{tabular}{@{}lccccc@{}}
\toprule
\textbf{Method} & \textbf{CHAIR$_S$$\downarrow$} & \textbf{CHAIR$_I$$\downarrow$} & \textbf{MCQ1$\uparrow$} & \textbf{MCQ2$\uparrow$} & \textbf{MCQ3$\uparrow$} \\
\midrule
Baseline & 0.328 & 0.080 & 0.248 & 0.432 & 0.202 \\
Nullu    & 0.284 & 0.073 & 0.244 & 0.415 & 0.197 \\
Ours     & \textbf{0.266} & \textbf{0.064} & \textbf{0.301} & \textbf{0.510} & \textbf{0.252} \\
\bottomrule
\end{tabular}
\caption{Cross-domain transfer on Gemma3-4B-it. CHAIR columns: filter learned from HaluEval, evaluated on CHAIR. MCQ columns: filter learned from LURE, evaluated on TruthfulQA.}
\label{tab:cross_domain}
\end{table}
\newpage

{
\section{LLM Judge Prompt for FaithDial}
\label{app:judge_prompt}
The following prompt was used for the Claude-based LLM judge in the FaithDial evaluation:

\begin{quote}
\small
\texttt{``You are a fact-checker evaluating whether a language model's dialogue response is grounded.
The model was given a knowledge passage and dialogue history, and must respond based on that context.
GROUNDING RULE: Every factual claim in the response must be supported by EITHER: 1. The knowledge passage, OR 2. The dialogue history (things already said in the conversation).
The model is allowed to: Paraphrase or summarize facts from the knowledge or dialogue history, Refer back to things already discussed in the conversation, Add pure conversational filler with no factual content.
The model is NOT allowed to: Add NEW facts from its own knowledge that appear in neither the knowledge passage nor the dialogue history, Invent specific details not in the context, Generate broken/repetitive text.
Definitions: 1. FAITHFUL: All factual claims come from the knowledge passage or dialogue history. 2. HALLUCINATION: The response introduces facts found in NEITHER the knowledge NOR the dialogue history, or the output is broken/degenerate.
Knowledge: \{knowledge\} Dialogue History: \{history\} Gold Response: \{gold\_response\} Model's Response: \{model\_output\}
Output your response in valid JSON format exactly like this: \{\{ ``is\_hallucination'': true or false, ``category'': ``Faithful'' | ``Hallucination'', ``reasoning'': ``Brief explanation'' \}\}
Return ONLY the JSON object, no other text.''}
\end{quote}
}

\newpage

\section{Assumption Validation}
\label{sec:assumption_validation}
To directly validate the weak cross-covariance assumption underlying our Wiener estimator, we compute assumption diagnostics on Gemma3-4B-it using paired truthful/hallucinated LURE representations over layers 12--26. We measure: (1) normalized cross-covariance $\rho_{\text{cross}} = \|\boldsymbol{\Sigma}_{sn}\|_F / \sqrt{\|\boldsymbol{\Sigma}_T\|_F \|\boldsymbol{\Sigma}_H\|_F}$, (2) covariance additivity error $\delta_{\text{add}} = \|\boldsymbol{\Sigma}_{x^+} - (\boldsymbol{\Sigma}_T + \boldsymbol{\Sigma}_H)\|_F / \|\boldsymbol{\Sigma}_{x^+}\|_F$, and (3) top-$k$ residual concentration $r_k = \mathbb{E}|\mathbf{Q}_k^\top \mathbf{d}_i|_2^2 / \mathbb{E}|\mathbf{d}_i|_2^2$. We also include a random-pair control where truthful partners are shuffled across images, breaking semantic pairing. Results are shown in Table~\ref{tab:assumption_diag}.

\begin{table}[t]
\centering
\small
{
\begin{tabular}{@{}lcccc@{}}
\toprule
\textbf{Setting} & $\rho_{\text{cross}}\downarrow$ & $\delta_{\text{add}}\downarrow$ & $r_{16}$ & $\|\mathbf{A}^\star\|_F$ \\
\midrule
Paired      & 0     & 0.099 & 0.948 & 155.917 \\
Random pair & 0     & 1.413 & 0.807 & 346.525 \\
\bottomrule
\end{tabular}
}
\caption{{Assumption diagnostics on Gemma3-4B-it (layers 12--26). The paired setting yields low additivity error and high spectral concentration, while random pairing breaks both, confirming that the filter depends on meaningful image-aligned pairs.}}
\label{tab:assumption_diag}
\end{table}

The paired setting yields a low additivity error of $0.099$, and the top 16 hallucination modes explain $94.8\%$ of the residual energy, indicating highly structured (non-isotropic) residuals. Random pairing raises the additivity error to $1.413$, reduces $r_{16}$ to $0.807$, and inflates the Wiener operator norm by a relative factor of $1.222$, confirming that the filter depends on meaningful image-aligned pairs rather than arbitrary covariance shrinkage.

\section{CHAIR Qualitative Results}
\label{app:chair_qual}
Figure~\ref{fig:qual_3_chair} presents representative captioning examples comparing greedy decoding, Nullu~\citep{yang2025nullu}, and our Wiener representation edit on LLaVA-1.5.
Across these examples, baseline approaches frequently introduce objects or attributes that are not supported by the visual evidence.
For instance, in the beach scene, greedy decoding hallucinates \emph{palm trees} and additional people sitting under umbrellas, while Nullu incorrectly describes the person as \emph{carrying a surfboard}.
Our method instead produces a grounded description that correctly reflects the visible elements in the image.

Similar behavior appears in the indoor dog example.
Greedy decoding incorrectly states that the dog is \emph{looking at the camera}, while Nullu misidentifies the furniture as a \emph{couch}.
In contrast, our caption correctly identifies the dog sitting on a bed and looking toward the window, consistent with the visual content.

Finally, in the train example, greedy decoding incorrectly describes the vehicle as a \emph{toy train}, and Nullu introduces \emph{people in the scene} that are not visible.
Our method produces a more faithful description that focuses on the visible steam train and surrounding environment without introducing unsupported entities.

Overall, these examples illustrate that hallucinations often arise as spurious object insertions or attribute misidentifications.
By suppressing hallucination-prone representation directions, the proposed Wiener edit produces captions that remain visually grounded while preserving descriptive quality.
\begin{figure*}[h]
\centering
\includegraphics[width=0.95\textwidth]{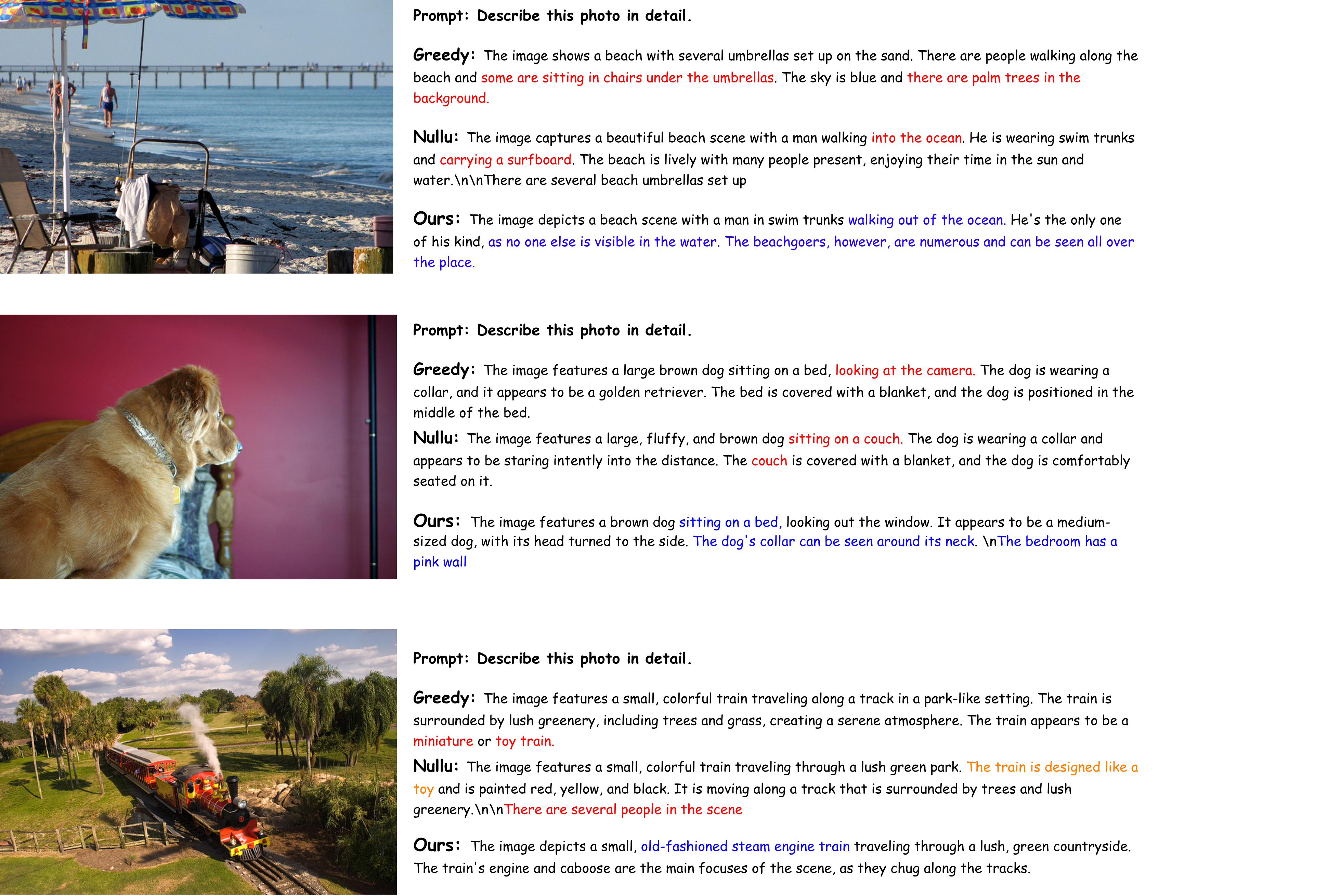}
 \vspace{-4pt}
 \caption{CHAIR captioning examples showing how Nullu and greedy decoding introduce spurious objects, while our method consistently remains visually grounded. Results are over the LLaVA 1.5 7B model.}
 \label{fig:qual_3_chair}
\vspace{-7pt}
\end{figure*}

\section{Runtime and Memory Profiling}
\label{app:runtime}
The filter is computed offline and then absorbed into the FFN output projection weights, so it introduces no additional inference-time memory or runtime cost. To quantify the offline overhead, we profiled filter construction on Gemma3-4B-it. The computation required approximately 2.441 seconds per edited layer, with a maximum per-layer CUDA peak memory increase of 154.01 MB. Since layers are processed sequentially, this memory overhead is not accumulated across layers. This cost is incurred only once during the offline weight-editing stage.

\end{document}